\documentclass{article}

\usepackage[final]{neurips_2021}

\usepackage[utf8]{inputenc} 
\usepackage[T1]{fontenc}    
\usepackage{booktabs}      
\usepackage{nicefrac}      
\usepackage{microtype}      
\usepackage{graphicx}
\usepackage{amsmath}
\usepackage{amsthm}
\usepackage{amsfonts}
\usepackage{amssymb}
\usepackage{mathtools}
\usepackage{subfigure}
\usepackage{subfigure}
\usepackage{colortbl}
\usepackage{xcolor}
\usepackage{enumitem}
\usepackage{bm}
\usepackage{lineno}
\usepackage{textcomp}
\usepackage{gensymb}
\usepackage{flushend}
\usepackage{multirow}
\usepackage{times}
\usepackage{natbib}
\usepackage[font=it]{caption}
\usepackage{doi}
\usepackage{placeins}
\usepackage{psfrag}
\usepackage{siunitx}
\usepackage{eurosym, units}
\usepackage{xtab}
\usepackage{verbatim}
\usepackage{pgfplots}
\usepackage{epstopdf}
\usepackage{float}
\usepackage{circuitikz}
\usepackage[export]{adjustbox}

\newcommand{\rulesep}{\unskip\ \vrule\ }
 
\newtheorem{theorem}{Theorem}
\newtheorem{lemma}[theorem]{Lemma} 
 
\newtheorem{remark}[theorem]{Remark}
\newtheorem{corollary}[theorem]{Corollary}
\newtheorem{definition}[theorem]{Definition}

\usepackage{enumitem}

\title{On the Stochastic Stability of Deep Markov Models}

\author{J\'an Drgo\v na$^1$, Sayak Mukherjee$^1$,   Jiaxin Zhang$^2$, Frank Liu$^2$, Mahantesh Halappanavar$^1$ \\
    $^1$ Pacific Northwest National Laboratory\\
	Richland, Washington, USA\\
	    $^2$ Oak Ridge National Laboratory\\
	Oak Ridge, Tennessee, USA\\ 
 \texttt{\{jan.drgona, sayak.mukherjee, mahantesh.halappanavar\}@pnnl.gov} \\
  \texttt{\{zhangj, liufy\}@ornl.gov}
}

\begin{document}

\maketitle

\begin{abstract}
Deep Markov models (DMM) are generative models that are scalable and expressive generalization of Markov models for representation, learning, and inference problems. However, the fundamental stochastic stability guarantees of such models have not been thoroughly investigated. In this paper, we provide sufficient conditions of DMM's stochastic stability as defined in the context of dynamical systems and propose a stability analysis method based on the contraction of probabilistic maps modeled by deep neural networks. We make connections between the spectral properties of neural network's weights and different types of used activation functions on the stability and overall dynamic behavior of DMMs with Gaussian distributions. Based on the theory, we propose a few practical methods for designing constrained DMMs with guaranteed stability. We empirically substantiate our theoretical results via intuitive numerical experiments using the proposed stability constraints.

\end{abstract}

\section{Introduction}
\label{sec:intro}
Modeling, analysis, and control of dynamical systems are of utmost importance for various physical and engineered systems such as fluid dynamics, oscillators, power grids, transportation networks, and autonomous driving, to name just a few. The systems are generally subjected to uncertainties arising from a plethora of factors such as exogenous noise, plant-model mismatch, and unmodeled system dynamics, which have led researchers to model the dynamics in stochastic frameworks. One of the most commonly used probabilistic frameworks to model dynamical system is the Hidden Markov Models (HMMs) \citep{rabiner1986introduction,eddy1996hidden} which in their vanilla form have been extensively investigated for representation, learning, and inference problems \citep{ghahramani1997factorial,cappe2006inference,beal2002infinite}. One of its variants, the Gaussian state space models, have been used in the systems and control community for decades \citep{beckers2016stability,eleftheriadis2017identification}. 

It has been shown that expressivity of Markov models to emulate complex dynamics and sequential behaviors is greatly improved by parametrizing such models using deep neural networks, giving rise to Deep Markov Models (DMMs)~ \citep{krishnan2016structured}.
Main research activities have been focused on the inference of these models. In particular, works such as \cite{awiszus2018markov, liu2019powering,mustafa2019comparative,qu2019gmnn} proposed parametrizing the probability distributions using deep neural networks for modeling various complex dynamical systems. 
Despite the rising popularity of DMMs,  many of their theoretical properties  such as robustness to perturbations, and stability of the generated trajectories remain open research questions.
Many natural systems exhibit complex yet stable dynamical behavior that is described by converging trajectories towards an attractor set~\citep{Brayton1979}. Additionally, safety-critical systems such as an autonomous driving call for formal verification methods to ensure safe operation.
Thus the ability to guarantee stability in data-driven models could lead to improved generalization or act as a safety certificate in real-world applications.

In this paper, we propose a new analytical method to assess stochastic stability of DMMs.
More specifically, we utilize spectral analysis of deep neural networks (DNNs) modeling DMM's distributions. 
This allows us to make connections between the stability of deterministic DNNs and the stochastic stability of deep Markov models.
As a main theoretical contribution we provide sufficient conditions for stochastic stability of DMMs. Based on the proposed theory we introduce several practical methods for design of constrained DMM with stability guarantees.
In summary, the main contributions of this paper include:
\begin{enumerate}[topsep=0pt,itemsep=-1ex,partopsep=1ex,parsep=1ex]
    \item \textbf{Stability analysis method for deep Markov models:} we base the analysis on the operator norms of deep neural networks modeling mean and variance of the DMM's distributions.
    \item \textbf{Sufficient conditions for stochastic stability of deep Markov models:} we show the sufficiency of the operator norm-based contraction conditions for DMM's deep neural networks by leveraging Banach fixed point theorem.
    \item \textbf{Stability constrained deep Markov models:} 
    we introduce a set of methods for the design of provably stable deep Markov models.
    \item \textbf{Numerical case studies:} we analyze connections between the design parameters of neural networks, stochastic stability, and operator norms of deep Markov models.
\end{enumerate}

\section{Related Work}
\label{sec:literature}

Deep Markov models (DMMs) have been used as a
scalable and expressive generalization of Hidden Markov Models (HMM) for learning probabilistic generative models of complex high-dimensional dynamical systems from sequential data~\citep{rezende2014stochastic,krishnan2016structured,fraccaro2016sequential}.
DMMs have been successfully applied to speech recognition problems~\citep{Li_dmm2013,prasetio_deep_2020},
control problems~\citep{ShashuaM17},
human pose forecasting~\citep{toyer2017human}, fault detection~\citep{WANG201865}, climate data forecasting~\citep{pmlr-v80-che18a}, molecular dynamics~\citep{wu2019deep}, or as internal models in model-based deep RL applied to automatic trading~\citep{ferreira2020reinforced}.
Several modifications of DMMs have been proposed to handle
incomplete data~\citep{Tan_DMM_2019}, or
multi-rate time series~\citep{pmlr-v80-che18a} 
 multivariate time series \citep{montanez2015inertial}, training DMMs in unsupervised settings~\citep{tran-etal-2016-unsupervised}, or architectures inspired by Kalman Filters~\citep{krishnan2015deep,ShashuaM17,RKN2019}.
 However, works focusing on formal analysis to ensure stability guarantees for DMMs are missing.

Stability notions and analysis for stochastic dynamic systems have been studied in the automatic control literature in various forms, depending on the representation of the system dynamics. Some classical results on stochastic stability for analysis and control can be found in \cite{mclane1971optimal, willems1976feedback}. These results are presented mainly for stochastic differential equation (SDE) models. \cite{Khasminskii2012} discusses different notions of stochastic stability, among them mean-square based stability notions have gained interest in works such as  
\cite{lu2002mean,farmer2009understanding,elia2013mean,nandanoori2018mean,wu2019switched}. We resort to such mean-square based stability notions when analyzing probabilistic state transition models parametrized by deep neural networks. We show that when the stochastic transitions are modeled by DNNs, the probabilistic stability requirements can be translated to deterministic stability notions of nonlinear discrete-time dynamics \citep{khalil}.

In recent years, deep neural networks have been extensively studied from the viewpoint of dynamical systems~\citep{NIPS2018_7892_NeuralODEs,Raissi_PINN_p1_2017,NAISnet2018}, allowing for the
application of stability analysis methods to DNNs.
For instance, \cite{NIPS2019_9292} proposed neural Lyapunov function to stabilize learned neural dynamics of  autonomous system.
\cite{haber2017stable} interpret residual connections in neural networks as Euler discretization of ODEs 
and provide stability guarantees of ResNet.
\cite{goel2017eigenvalue} makes connections between eigenvalue decay and learnability of neural networks.
\cite{engelken2020lyapunov,vogt2020lyapunov} studies the Lyapunov spectrum to of the input-output Jacobian of recurrent neural networks (RNNs) to assess RNN's stability.
In this work we leverage advances in the spectral analysis of deep neural networks and apply them to derive stochastic stability guarantees for DMMs.

Besides analytical methods, many authors introduced provably stable neural architectures~\citep{IMEXnet2019,HamiltonianDNNs2019,cranmer2020lagrangian} or stability  constraints~\citep{john2017gabor}.
Another popular strategy is to employ stabilizing regularizations. 
This can be achieved by minimizing eigenvalues of squared weights~\citep{ludwig2014eigenvalue},  
using symplectic weights~\citep{haber2017stable}, orthogonal parametrizations~\citep{mhammedi2017efficient}, Perron-Frobenius theorem~\citep{tuor2020constrained}, Gershgorin dics theorem~\citep{lechner2020gershgorin},
or via singular value  decomposition (SVD)~\citep{zhang2018stabilizing}.
In this paper we leverage different weight factorization to empirically validate the proposed theoretical guarantees on stochastic stability of DMMs.




\section{Methodology}
\label{sec:methods}
This section presents stochastic stability analysis method for deep Markov models (DMM). 
First, we demonstrate the equivalence of DNNs with pointwise affine (PWA) functions.
Next, we  recall the definition of DMM with transition probabilities modeled  by deep neural networks (DNNs).
 We introduce definitions of stochastic stability and show how can we leverage deterministic stability analysis in the probabilistic context. 
Finally, we will leverage the equivalence of DNNs with PWA maps to pose sufficient stability stochastic conditions for DMMs based on contraction of PWA maps.

\subsection{Equivalence of Deep Neural Networks with Pointwise Affine Maps}
Let us consider deep neural network (DNN)
$\boldsymbol \psi: \mathbb{R}^n \rightarrow \mathbb{R}^m$ parametrized by  $\theta_{\boldsymbol \psi} = \{\mathbf{A}_0^{\boldsymbol \psi}, \ldots \mathbf{A}_{L}^{\boldsymbol \psi}, \mathbf{b}_{0}, \ldots, 
\mathbf{b}_{L}\}$ 
with hidden layers $1\leq l\leq L$ with bias given as follows:
\begin{subequations}
    \label{eq:dnn_map}
    \begin{align}
    \boldsymbol \psi_{\theta_{\boldsymbol \psi}}(\mathbf{x}) & =  \mathbf{A}_{L}^{\boldsymbol \psi} \mathbf{h}_L^{\boldsymbol \psi} + \mathbf{b}_L\\
    \mathbf{h}_{l}^{\boldsymbol \psi} &= \boldsymbol v(\mathbf{A}_{l-1}^{\boldsymbol \psi} \mathbf{h}_{l-1}^{\boldsymbol \psi} + \mathbf{b}_{l-1})
 \end{align}
\end{subequations}
 with $ \mathbf{h}_0^{\boldsymbol \psi} = \mathbf{x}$, and  $\boldsymbol v: \mathbb{R}^{n_z} \rightarrow \mathbb{R}^{n_z}$ representing element-wise application of an activation function to vector elements such that $\boldsymbol v(\mathbf{z}) := \begin{bmatrix} \boldsymbol v(z_1)\hdots  \boldsymbol v(z_{n_z})\end{bmatrix}^{\intercal}$. 
 \begin{lemma}
 \label{lem:lpv}
 For a multi-layer feedforward neural network $\boldsymbol \psi_{\theta_{\boldsymbol \psi}}$~\eqref{eq:dnn_map} with arbitrary activation
function $\boldsymbol v$, there exists an equivalent pointwise affine map (PWA) parametrized by $\mathbf{x}$ which satisfies:
\begin{align}
 \label{eq:lpv_dnn}
    \boldsymbol \psi_{\theta_{\boldsymbol \psi}}(\mathbf{x}) = \mathbf{A}_{\boldsymbol \psi}(\mathbf{x}) \mathbf{x} + \mathbf{b}_{\boldsymbol \psi}(\mathbf{x}).
\end{align}

Where  $\mathbf{A}_{\boldsymbol \psi}(\mathbf{x})$ is a parameter varying matrix given as:
\begin{equation}
  \label{eq:dnn_LPV}
   \mathbf{A}_{\boldsymbol \psi}(\mathbf{x})\mathbf{x}= \mathbf{A}^{\boldsymbol \psi}_L \boldsymbol\Lambda^{\boldsymbol \psi}_{\mathbf{z}_{L-1}} \mathbf{A}^{\boldsymbol \psi}_{L-1} \ldots \boldsymbol \Lambda^{\boldsymbol \psi}_{\mathbf{z}_{0}} \mathbf{A}^{\boldsymbol \psi}_0 \mathbf{x}
 \end{equation}
And $\mathbf{b}_{\boldsymbol \psi}(\mathbf{x})$ is a parameter varying vector, both parametrized by input vector $\mathbf{x}$ given by following recurrent formula:
  \begin{eqnarray}
  \label{eq:dnn_bias_recurence}
 \mathbf{b}_{\boldsymbol \psi}(\mathbf{x}) =  \mathbf{b}^{\boldsymbol \psi}_{L}  \\
  \mathbf{b}^{\boldsymbol \psi}_{l} := \mathbf{A}^{\boldsymbol \psi}_i \boldsymbol\Lambda^{\boldsymbol \psi}_{\mathbf{z}_{l-1}} \mathbf{b}^{\boldsymbol \psi}_{l-1} +\mathbf{b}_{l} +\boldsymbol\sigma_{l-1}(\mathbf{0}), \ \ l \in \mathbb{N}_1^{L}
 \end{eqnarray}
  with $ \mathbf{b}^{\boldsymbol \psi}_{0} = \mathbf{b}_{0}$, and  $i$ representing index of the network layer.
 Here $\boldsymbol \Lambda^{\boldsymbol \psi}_{\mathbf{z}_{l}}$ represents parameter varying diagonal matrix of activation patterns defined as:
   \begin{equation}
\label{eq:lambda_matrix}
\boldsymbol\sigma(\mathbf{z})  =  \begin{bmatrix}
   \frac{\sigma(z_1) -\sigma(0)}{z_1} &  & \\
    & \ddots & \\
     &  &  \frac{\sigma(z_n)-\sigma(0)}{z_n} 
  \end{bmatrix}\mathbf{z} + \begin{bmatrix}\sigma(0)\\\vdots \\\sigma(0)\end{bmatrix} = \boldsymbol \Lambda^{\boldsymbol \psi}_{\mathbf{z}}  \mathbf{z} + \boldsymbol\sigma(\mathbf{0}) 
\end{equation} 
 \end{lemma}
 
  \begin{proof}
 First lets observe the following:
  \begin{equation}
\label{eq:lambda_matrix_steps}
\boldsymbol\sigma(\mathbf{z})  =  \begin{bmatrix}\sigma(z_1)\\\vdots \\\sigma(z_n)\end{bmatrix} = 
 \begin{bmatrix}\frac{z_1 (\sigma(z_1)-\sigma(0)+\sigma(0))}{z_1} \\\vdots\\ \frac{z_n (\sigma(z_n)-\sigma(0)+\sigma(0))}{z_n}\end{bmatrix}  =  \begin{bmatrix}
   \frac{\sigma(z_1) -\sigma(0)}{z_1} &  & \\
    & \ddots & \\
     &  &  \frac{\sigma(z_n)-\sigma(0)}{z_n} 
  \end{bmatrix}\mathbf{z} + \begin{bmatrix}\sigma(0)\\\vdots \\\sigma(0)\end{bmatrix} 
\end{equation} 
Remember  $\sigma(0)-\sigma(0) = 0$, and $\frac{z_i}{z_i} = 1$ are identity elements of addition and multiplication, respectively.
Thus~\eqref{eq:lambda_matrix_steps} demonstrates the equivalence $\boldsymbol\sigma(\mathbf{z}) = \boldsymbol \Lambda^{\boldsymbol \psi}_{\mathbf{z}} \mathbf{z} + \boldsymbol\sigma(\mathbf{0})$  as given in~\eqref{eq:lambda_matrix}.
Then if we let $\mathbf{z}_{l} =\mathbf{A}^{\boldsymbol \psi}_l\mathbf{x}_l + \mathbf{b}_{l}$, we can represent a  neural network layer in a parameter varying affine form:
\begin{equation}
 \label{eq:dnn_LPV_layer}
\boldsymbol\sigma_l(\mathbf{A}^{\boldsymbol \psi}_l\mathbf{x}_l  + \mathbf{b}_{l}) =  \boldsymbol \Lambda^{\boldsymbol \psi}_{\mathbf{z}_{l}} (\mathbf{A}^{\boldsymbol \psi}_l\mathbf{x}_l  + \mathbf{b}_{l}) + \boldsymbol\sigma(\mathbf{0}) =  \boldsymbol \Lambda^{\boldsymbol \psi}_{\mathbf{z}_{l}}\mathbf{A}^{\boldsymbol \psi}_l\mathbf{x}_l  + \boldsymbol \Lambda^{\boldsymbol \psi}_{\mathbf{z}_{l}}\mathbf{b}_{l} + \boldsymbol\sigma_l(\mathbf{0})
 \end{equation}
 Now for simplicity of exposition lets assume only activations with trivial null space, i.e. $\sigma(0) = 0$. Thus $\boldsymbol\sigma(\mathbf{z}) = \boldsymbol \Lambda^{\boldsymbol \psi}_{\mathbf{z}} \mathbf{z}$.
By composition, a DNN $ \boldsymbol \psi_{\theta_{\boldsymbol \psi}}(\mathbf{x})$ can now be formulated as a parameter-varying affine map
 $\mathbf{A}_{\boldsymbol \psi}(\mathbf{x}) \mathbf{x} + \mathbf{b}_{\boldsymbol \psi}(\mathbf{x})$, parametrized by input $\mathbf{x}$
\begin{equation}
\begin{aligned}
  \label{eq:dnn_LPV_bias}
     &   \boldsymbol \psi_{\theta_{\boldsymbol \psi}}(\mathbf{x}) :=   \mathbf{A}_{\boldsymbol \psi}(\mathbf{x}) \mathbf{x} + \mathbf{b}_{\boldsymbol \psi}(\mathbf{x}) = \\
    &  \mathbf{A}^{\boldsymbol \psi}_L \boldsymbol\Lambda^{\boldsymbol \psi}_{\mathbf{z}_{L-1}}( \mathbf{A}^{\boldsymbol \psi}_{L-1} ( \ldots   \boldsymbol\Lambda^{\boldsymbol \psi}_{\mathbf{z}_{1}}( \mathbf{A}^{\boldsymbol \psi}_1 \boldsymbol\Lambda^{\boldsymbol \psi}_{\mathbf{z}_{0}}(\mathbf{A}^{\boldsymbol \psi}_0\mathbf{x} + \mathbf{b}_0) + \mathbf{b}_1) \ldots ) + \mathbf{b}_{L-1})\mathbf{x} + \mathbf{b}_L  \\
      & \mathbf{A}_{\boldsymbol \psi}(\mathbf{x}) \mathbf{x} = \mathbf{A}^{\boldsymbol \psi}_L \boldsymbol\Lambda^{\boldsymbol \psi}_{\mathbf{z}_{L-1}} \mathbf{A}^{\boldsymbol \psi}_{L-1} \ldots \boldsymbol \Lambda^{\boldsymbol \psi}_{\mathbf{z}_{0}} \mathbf{A}^{\boldsymbol \psi}_{0} \mathbf{x}    \\
      & \mathbf{b}_{\boldsymbol \psi}(\mathbf{x}) = \mathbf{A}^{\boldsymbol \psi}_{L}   \ldots  \mathbf{A}^{\boldsymbol \psi}_{2} \boldsymbol \Lambda^{\boldsymbol \psi}_{\mathbf{z}_{1}} \mathbf{A}^{\boldsymbol \psi}_{1} \boldsymbol \Lambda^{\boldsymbol \psi}_{\mathbf{z}_{0}}\mathbf{b}_{0} +  \mathbf{A}^{\boldsymbol \psi}_{L}   \ldots  \mathbf{A}^{\boldsymbol \psi}_{2} \boldsymbol \Lambda^{\boldsymbol \psi}_{\mathbf{z}_{1}} \mathbf{A}^{\boldsymbol \psi}_{1}  \boldsymbol\sigma_{0}(\mathbf{0}) +\\
      & \mathbf{A}^{\boldsymbol \psi}_{L}   \ldots   \mathbf{A}^{\boldsymbol \psi}_{2} \boldsymbol \Lambda^{\boldsymbol \psi}_{\mathbf{z}_{1}} \mathbf{b}_{1} +  \mathbf{A}^{\boldsymbol \psi}_{L}   \ldots   \mathbf{A}^{\boldsymbol \psi}_{2} \boldsymbol\sigma_{1}(\mathbf{0}) + \ldots +  \mathbf{A}^{\boldsymbol \psi}_{L} \boldsymbol \Lambda^{\boldsymbol \psi}_{\mathbf{z}_{L-1}} \mathbf{b}_{L-1} + \mathbf{A}^{\boldsymbol \psi}_{L}  \boldsymbol\sigma_{L-1}(\mathbf{0}) + \mathbf{b}_{L}
\end{aligned}
 \end{equation}
Hence, each input feature vector $\mathbf{x}$ generates
a unique affine map $ \mathbf{A}_{\boldsymbol \psi}(\mathbf{x}) \mathbf{x} + \mathbf{b}_{\boldsymbol \psi}(\mathbf{x})$  of the DNN $  \boldsymbol \psi_{\theta_{\boldsymbol \psi}}(\mathbf{x}) $. Thus proving the equivalence of DNN map~\eqref{eq:dnn_map} with the form~\eqref{eq:lpv_dnn}.
The case with $\boldsymbol\sigma(\mathbf{0}) \neq 0$ can be derived following the same algebraic operations as as above.
\end{proof}

\subsection{Deep Markov Models} \label{sec:DMMs}

We consider a dynamical system with latent state variables $\mathbf{x}_t \in \mathbb{R}^n$, and the observed variables $\mathbf{y}_t \in \mathbb{R}^m$. The transition from $\mathbf{x}_t$ to the next time step $\mathbf{x}_{t+1}$, and the outputs $\mathbf{y}_t$ are modeled by probabilistic transitions. Over a horizon of $T$ time steps with a step size $\Delta t$, we assume the Markov property to embed structural independence conditions in the dynamic state evolution, i.e.,
\begin{align} \label{markov_prop}
    \mathbf{x}_{t+1} \perp \mathbf{x}_{0:t-1} \; | \; \mathbf{x}_{t},
\end{align}
Thus having latent state dynamics characterized by one-time-step conditional distribution $P(\mathbf{x}_{t+1} | \mathbf{x}_t)$. The joint distribution over the latent states and the observations is given by,
\begin{align}\label{joint_dist}
    P(\mathbf{x}_{0:T}, \mathbf{y}_{0:T}) = P(\mathbf{x}_0) P(\mathbf{y}_0 |\mathbf{x}_0)\prod_{t=0}^{T-1} P(\mathbf{x}_{t+1} | \mathbf{x}_t)P(\mathbf{y}_t | \mathbf{x}_t).
\end{align}
More explicitly, we consider the following probabilistic transition and the emission maps:
\begin{subequations}
    \label{eq:DMM}
\begin{align}\label{eq:transition}
    \mathbf{x}_{t+1} & \sim \mathcal{N}(K_{\alpha}(\mathbf{x}_t, \Delta t), L_{\beta}(\mathbf{x}_t, \Delta t)) & \text{(Transition)}\\
    \mathbf{{y}}_{t} & \sim \mathcal{M} (F_{\kappa}(\mathbf{x}_{t})) & \text{(Emission)}
\end{align}
\end{subequations}
with the initial condition $\mathbf{x}_{t=0} = \mathbf{x}_0.$  Here,
$\mathcal{N}$ and $\mathcal{M}$ denote the probability distributions. For the transition mapping, $\mathcal{N}$ denotes a Gaussian distribution with the mean vector $K_{\alpha}(\mathbf{x}_t, \Delta t)$, and covariances $L_{\beta}(\mathbf{x}_t, \Delta t)$. The distribution $\mathcal{M}$ can be arbitrary with its distribution characterized by the map $F_{\kappa}(\mathbf{x}_{t})$. In this paper, we are interested in the stability characterizations of the latent state dynamics given by transition map~\eqref{eq:transition}, thereby assuming full state observability.  
 
We are interested in expressing the dynamics of complex systems using \eqref{eq:DMM}, therefore it is suitable to expand the expressivity of our model by parametrizing the conditional distribution $P(\mathbf{x}_{t+1} | \mathbf{x}_t)$ by deep neural networks (DNNs) given as,
\begin{align}
    K_{\alpha}(\mathbf{x}_{t}, \Delta t) = \mathbf{f}_{\theta_{\mathbf{f}}}(\mathbf{x}_t), \label{eq:dnn_dyn} \\
    \mbox{vec}(L_{\beta}(\mathbf{x}_{t}, \Delta t)) = \mathbf{g}_{\theta_{\mathbf{g}}}(\mathbf{x}_t), \label{eq:dnn_dyn2}
\end{align}
where $\mathbf{f}:  \mathbb{R}^n \rightarrow \mathbb{R}^n$ and $\mathbf{g}: \mathbb{R}^n \rightarrow \mathbb{R}^{n^2}$ are two deep neural networks~\label{eq:dnn_map} parametrized by $\theta_{\mathbf{f}}$ and $\theta_{\mathbf{g}}$, respectively. And $\mbox{vec}(\cdot)$ denotes standard vectorization operation. Therefore, the probabilistic transition dynamics \eqref{eq:transition} can be characterized by analysing the stability and boundedness of deep neural networks $\mathbf{f}$ and $\mathbf{g}$.

\subsection{Stability of Deep Markov Models} \label{sec:Stability_DMM}

To this end, we bring forth a few stability notions in the context of stochastic state transitions. In stochastic dynamics and control literature \citep{Khasminskii2012}, various different notions of stability for stochastic state transitions, such as mean-square stability, almost-sure stability and stability via convergence in probability, have been discussed. In this article, since we are interested in the latent state trajectories of the dynamic systems, we consider the mean-square stability as defined in~\citep{willems1976feedback,nandanoori2018mean}. The dynamic system is said to be mean-square stable if the first and the second moment converge over time. 
\begin{definition}\label{def:MSS}
 The stochastic process $\mathbf{x}_t \in \mathbb{R}^n$ is mean-square stable (MSS) if and only if there exists $\mathbf{\mu} \in \mathbb{R}^n, \mathbf{\Sigma} \in \mathbb{R}^{n \times n},$ such that $\lim_{t \to \infty}\mathbb{E}(\mathbf{x}_t) = \mathbf{\mu}$, and $\lim_{t \to \infty}\mathbb{E}(\mathbf{x}_t\mathbf{x}_t^T) = \mathbf{\Sigma}$.
\end{definition}

The MSS condition from Definition \ref{def:MSS} requires the dynamics \eqref{eq:dnn_dyn} to have stable equilibrium $\bar{\mathbf{x}}_e = \mathbf{\mu}$, where $\bar{\mathbf{x}}$ denotes the mean state vector. The definition also requires the covariance to converge. To express the dynamic behavior of complex system, the second moment convergence criterion can be relaxed by only requiring it to be norm bounded in order to ensure stochastic stability given as,
\begin{equation}
\label{eq:g_bounded}
    \|{\mathbf{g}_{\theta_\mathbf{g}}(\mathbf{x}_t)}\|_p < K, \ K>0, \ \forall t.
\end{equation}
Here $\|\cdot \|_p$ denotes any appropriate vector norm, e.g., $L2-$norm. 
The bound on the covariances will depend on the extent of stochasticity that the dynamic system encounters in an uncertain environment. However, in the later parts, we consider the convergence scenario for $\mathbf{g}$ as in Definition \ref{def:MSS}, rather than merely boundedness, in prescribing the sufficient conditions for the MSS-type stability. 

Let us first analyze the mean dynamics characterized by \eqref{eq:dnn_dyn}, and its equilibrium $\bar{\mathbf{x}}_e = \mu$ which 
satisfies the stationarity condition $\mathbf{f}_{\theta_\mathbf{f}}(\bar{\mathbf{x}}_e) = \bar{\mathbf{x}}_e$. We have the mean state vector $\bar{\mathbf{x}}_t$ evolving under the following  dynamics:
\begin{align}\label{eq:mean_dyn}
    \bar{\mathbf{x}}_{t+1} = \mathbf{f}_{\theta_\mathbf{f}}(\bar{\mathbf{x}}_t).
\end{align}
\eqref{eq:mean_dyn} allows us to analyze the asymptotic stability
of the DMM mean dynamics. 
Now in the main result of this section we leverage the fact that the dynamic characteristics of the deep neural networks $\mathbf{f}_{\theta_\mathbf{f}}$, $\mathbf{g}_{\theta_\mathbf{g}}$ around a point $\bar{\mathbf{x}}_t$ can be evaluated by obtaining their exact pointwise affine forms (PWA)~\eqref{eq:lpv_dnn}. 
Based on this equivalence we formulate Theorem~\ref{thm:DMM_stability_1} and Corollary~\ref{thm:DMM_stability_2} as follows with  sufficient conditions for the stability of deep Markov models. 

\begin{theorem}
\label{thm:DMM_stability_1}
The deep Markov model \eqref{eq:DMM} which is parametrized by deep neural networks  \eqref{eq:dnn_dyn}-\eqref{eq:dnn_dyn2} remains globally stable in the mean-square sense if the following  holds. 
  The mean neural network $\mathbf{f}_{\theta_\mathbf{f}}(\mathbf{x})$ 
 is a contractive map for any $\mathbf{x}$ in the domain of $\mathbf{f}_{\theta_\mathbf{f}}(\mathbf{x})$.  The variance network $\mathbf{g}_{\theta_\mathbf{g}}(\mathbf{x})$ is bounded for any $\mathbf{x}$ in the domain of $\mathbf{g}_{\theta_\mathbf{g}}(\mathbf{x})$. Or more formally:
    \begin{subequations}
        \label{eq:sufficient_1}
        \begin{align}
    &    \|\mathbf{A}_{\mathbf{f}}(\mathbf{x}) \|_p < 1 
     \label{eq:sufficient_1:1}  \\
    &  ||\mathbf{A}_{\mathbf{g}}(\mathbf{x})||_p+ \frac{||\mathbf{b}_{\mathbf{g}}(\mathbf{x})||_p}{\| \mathbf{x} \|_p} < 1, \label{eq:sufficient_1:2} \\ 
    &  \forall \mathbf{x} \in \text{Domain}(\mathbf{f}_{\theta_\mathbf{f}}(\mathbf{x}), \mathbf{g}_{\theta_\mathbf{g}}(\mathbf{x})).
        \end{align}
    \end{subequations} 
\end{theorem}

\begin{proof}
First we prove the sufficiency of the contraction condition of the mean dynamics~\eqref{eq:sufficient_1:1}.
We base the proof on the equivalence of multi-layer neural networks with pointwise affine maps~\eqref{eq:lpv_dnn}.
An affine map is a contraction 
if the $2$-norm of its linear part is bounded below one, i.e. $ ||\mathbf{A}||_2 < 1$.
Thus it follows that the condition~\eqref{eq:sufficient_1:1} and equivalence~\eqref{eq:lpv_dnn} imply a contractive mean neural network $\mathbf{f}_{\theta_\mathbf{f}}(\mathbf{x})$.
The sufficiency of the contraction condition on mean square stable (MSS) equilibrium in the sense of Definition~\ref{def:MSS} follows directly from the Banach fixed-point theorem, which states that every contractive map converges towards single point equilibirum. Hence condition~\eqref{eq:sufficient_1:1} implies convergent mean transition dynamics:
\begin{equation}
   \mathbf{\mu} = \mathbf{f}_{\theta_\mathbf{f}}(\mathbf{\mu}) = 
   \lim_{t \to \infty} \mathbf{f}_{\theta_\mathbf{f}}(\bar{\mathbf{x}}_t)
\end{equation}


Now we show the sufficiency of~\eqref{eq:sufficient_1:2}
to guarantee the boundedness of the covariance matrix elements~\eqref{eq:g_bounded}  by
bounding the $p$-norm of the covariance neural network $ \| \mathbf{g}_{\theta_\mathbf{g}}(\mathbf{x})\|_p$.
Please note that using the form~\eqref{eq:lpv_dnn} gives us $ \| \mathbf{g}_{\theta_\mathbf{g}}(\mathbf{x})\|_p  =
\|\mathbf{A}_{\mathbf{g}}(\mathbf{x})\mathbf{x} + \mathbf{b}_{\mathbf{g}}(\mathbf{x})\|_p$ yielding following inequalities:
\begin{subequations}
\begin{align}
  \label{eq:variance_norm_inequality}
&  \| \mathbf{g}_{\theta_\mathbf{g}}(\mathbf{x})\|_p \le 
\|\mathbf{A}_{\mathbf{g}}(\mathbf{x})\mathbf{x}\|_p  + \|\mathbf{b}_{\mathbf{g}}(\mathbf{x})\|_p, 
\\
&  \frac{||\mathbf{g}_{\theta_\mathbf{g}}(\mathbf{x})||_p}{\| \mathbf{x} \|_p} \le  \|\mathbf{A}_{\mathbf{g}}(\mathbf{x}) \|_p + \frac{||\mathbf{b}_{\mathbf{g}}(\mathbf{x})||_p}{\| \mathbf{x} \|_p}.  \label{eq:variance_norm_inequality:2}
\end{align}
\end{subequations}
We show that~\eqref{eq:variance_norm_inequality:2} gives in fact a local Lipschitz constant of the variance network $\mathbf{g}_{\theta_\mathbf{g}}(\mathbf{x})$.
We exploit the point-wise affine nature of the neural network's form~\eqref{eq:lpv_dnn} and the fact that the norm of a linear operator $\mathbf{A}$ is equivalent to its minimal Lipschitz constant $\mathcal{K}^{A}_{min} = || \mathbf{A} ||_p $~\citep{huster2018limitations}. 
Thus we can compute the local Lipschitz constants of a neural network $\mathbf{g}_{\theta_\mathbf{g}}(\mathbf{x})$ as:
\begin{align}
  \label{eq:dnn_global_Lipschitz_norm_bias}
    \mathcal{K}^{\mathbf{g}}(\mathbf{x}) =   ||\mathbf{A}_{\mathbf{g}}(\mathbf{x})||_p+ \frac{||\mathbf{b}_{\mathbf{g}}(\mathbf{x})||_p}{\| \mathbf{x} \|_p}.
\end{align}
Applying
the upper bound~\eqref{eq:sufficient_1:2} 
on the local Lipschitz constant~\eqref{eq:dnn_global_Lipschitz_norm_bias} 
 guarantees the contraction of the variance neural network $\mathbf{f}_{\theta_\mathbf{g}}(\mathbf{x})$ towards a fixed steady state $\mathbf{\Sigma}$.
\end{proof}

\begin{remark}
  To guarantee stochastic stability, the condition~\eqref{eq:sufficient_1:2} can be relaxed as given in~\eqref{eq:g_bounded} to bounded second moment~\eqref{eq:g_bounded} with $ \max_\mathbf{x} \mathcal{K}^{\mathbf{g}}(\mathbf{x}) < K$, where $K>0$. 
\end{remark}

\begin{corollary}
\label{thm:DMM_stability_2}
The deep Markov model \eqref{eq:DMM} which is parametrized by deep neural networks  \eqref{eq:dnn_dyn}-\eqref{eq:dnn_dyn2} remains globally stable in the mean-square sense if the following  holds: 
  All weights $\mathbf{A}_i^{\mathbf{f}}$  of the mean network are $\mathbf{f}_{\theta_\mathbf{f}}$  contractive maps. All activation scaling matrices $\boldsymbol\Lambda^{\mathbf{f}}_{\mathbf{z}_i}$  of the mean
  network are non-expanding.
  Norms of all weights $\mathbf{A}_j^{\mathbf{g}}$
  and activation scaling matrices 
  $\boldsymbol\Lambda^{\mathbf{g}}_{\mathbf{z}_j}$
  of the variance network $\mathbf{g}_{\theta_\mathbf{g}}$ 
  are upper bounded by $1$. Or more formally:
        \begin{subequations}
        \label{eq:sufficient_2}
        \begin{align}
       & \|\mathbf{A}_i^{\mathbf{f}}\|_p < 1,  \ ||\boldsymbol\Lambda^{\mathbf{f}}_{\mathbf{z}_i}||_p \le 1 \ i \in \mathbb{N}_1^{L_{\mathbf{f}}}, \\ 
       & \|\mathbf{A}_j^{\mathbf{g}}\|_p < 1,  \ ||\boldsymbol\Lambda^{\mathbf{g}}_{\mathbf{z}_j}||_p \le 1, \ j \in \mathbb{N}_1^{L_{\mathbf{g}}}, \label{eq:sufficient_2:2} \\ 
      &  \forall \mathbf{x} \in \text{Domain}(\mathbf{f}_{\theta_\mathbf{f}}(\mathbf{x}), \mathbf{g}_{\theta_\mathbf{g}}(\mathbf{x})).
        \end{align}
    \end{subequations}
\end{corollary}


\begin{proof}
First we show the sufficiency $\|\mathbf{A}_i\|_p < 1$
of contractive weights and non-expanding activation scaling matrices $\|\boldsymbol\Lambda_{\mathbf{z}_i}\| \le 1$
to guarantee the contractivity of arbitrary deep neural networks.
Assuming general non-square weights $\mathbf{A}_i \in \mathbf{R}^{n_i \times m_i}$ we use the submultiplicativity of the induced $p$-norms to upper bound the norm of a products of $m$ matrices given as:
\begin{equation}
\label{eq:Gelfand_norm}
 \|\mathbf{A}_1  \ldots \mathbf{A}_m  \|_p \le
     \| \mathbf{A}_1  \|_p  \ldots \| \mathbf{A}_m  \|_p
\end{equation}
Now by applying~\eqref{eq:Gelfand_norm}  
 to the linear parts~\eqref{eq:dnn_LPV} of the mean neural network $\mathbf{f}_{\theta_\mathbf{f}}$  in the pointwise affine form~\eqref{eq:lpv_dnn} with 
 $ \|\mathbf{A}_i^{\mathbf{f}} \|_p < 1, \ \forall i \in \mathbb{N}_0^L$, $\|\boldsymbol\Lambda^{\mathbf{f}}_{\mathbf{z}_j}\|_p   \le 1,  \ \forall j \in \mathbb{N}_1^L$, it yields $\|\mathbf{A}_{\mathbf{f}}(\mathbf{x})\|_p < 1$ over the entire domain of $\mathbf{f}_{\theta_\mathbf{f}}(\mathbf{x})$, thus with $p=2$ satisfying the contraction condition  $||\mathbf{A}_\mathbf{f}(\mathbf{x})||_2 < 1$ for affine maps.
 The submultiplicativity~\eqref{eq:Gelfand_norm}  naturally applies also to the variance network  $\mathbf{g}_{\theta_\mathbf{g}}(\mathbf{x})$ thus implying the contraction towards a fixed point given the conditions~\eqref{eq:sufficient_2}.
\end{proof}

\begin{remark}
   To guarantee stochastic stability, we can relax the upper bound of the second moment as $K = \prod_i^L c^{\mathbf{A}} c^{\boldsymbol\Lambda}$, where $c^{\mathbf{A}} >0$, and $c^{\boldsymbol\Lambda} >0$ represent the relaxed upper bounds of the operator norms in condition~\eqref{eq:sufficient_2:2}.  
 Thus satisfying the  relaxed boundedness condition on the variance via~\eqref{eq:g_bounded}.
\end{remark}

Assuming the contraction conditions~\eqref{eq:sufficient_1}
or~\eqref{eq:sufficient_2} hold and the mean neural network $\mathbf{f}_{\theta_\mathbf{f}}$ has zero bias, then the DMM's mean network $\mathbf{f}_{\theta_\mathbf{f}}(\mathbf{x})$ is equivalent with stable parameter varying linear map~\eqref{eq:dnn_LPV}
with equilibrium in the origin, i.e.  $\bar{\mathbf{x}} = \mathbf{0}$. 
In the case with non-zero bias in $\mathbf{f}_{\theta_\mathbf{f}}$,
the corresponding PWA map~\eqref{eq:lpv_dnn} has non-zero equilibrium $\bar{\mathbf{x}} \neq \mathbf{0}$.
Both conditions~\eqref{eq:sufficient_1}
or~\eqref{eq:sufficient_2} are sufficient for a convergence of a DMM~\eqref{eq:DMM} to a stable equilibrium $\bar{\mathbf{x}}$. However, they do not provide bounds of the admissible values of the equilibrium $\bar{\mathbf{x}}$.
The corresponding equilibrium bounds are provided in the supplementary material.

\subsection{Design of Stable Deep Markov Models}
\label{sec:design}
In this section, we provide a set of practical design methods for provably stable DMM~\eqref{eq:DMM}.
Based on the Corollary~\ref{thm:DMM_stability_2} the use of contractive activation functions together with contractive weights for both mean and variance network will guarantee the stability of DMM by design.
In particular, the conditions~\eqref{eq:sufficient_2}  on bounded norm of transitions' activation scaling matrices $||\boldsymbol\Lambda^{\mathbf{f}}_{\mathbf{z}_i}|| \le 1$,
$||\boldsymbol\Lambda^{\mathbf{g}}_{\mathbf{z}_j}|| \le 1$ implies Lipschitz continuous activation functions with constant $\mathcal{K} \le 1$.  Conveniently, this condition is satisfied for many popular activation functions such as \texttt{ReLU}, \texttt{LeakyReLU}, or \texttt{tanh}.
The contractivity conditions~\eqref{eq:sufficient_2}
on weight matrices $\|\mathbf{A}_i^{\mathbf{f}}\|_p <1$,
 $\|\mathbf{A}_j^{\mathbf{g}}\|_p <1$, respectively, 
 can be enforced by employing various matrix factorizations 
 proposed in the deep neural network litertature. 
 Examples include
  singular value decomposition (SVD)~\citep{zhang2018stabilizing},  
   Perron-Frobenius (PF)~\citep{tuor2020constrained}, and
   Gershgorin discs (GD)~\citep{lechner2020gershgorin}  factorizations given below.
   
\paragraph{PF weights:}
 This factorization applies Perron-Frobenius theorem for  constraining the dominant eigenvalue of the square nonnegative matrices. 
Based on this theorem, we can construct the weight matrix $\mathbf{{A}}$ with bounded eigenvalues as follows:
\begin{subequations}
\label{eq:pf}
\begin{align}
\mathbf{M} &= \lambda_{\text{max}} - (\lambda_{\text{max}} - \lambda_{\text{min}})  g(\mathbf{M'}) \\
\mathbf{{A}}_{i,j} &= \frac{\text{exp}(\mathbf{A'}_{ij})}{\sum_{k=1}^{n_x} \text{exp}(\mathbf{A'}_{ik})}\mathbf{M}_{i,j}
\end{align}
\end{subequations}
here $\mathbf{M}$ represents the damping factor parameterized by the matrix $\mathbf{M'} \in \mathbb{R}^{n_x \times n_x}$, while $\mathbf{A'} \in \mathbb{R}^{n_x \times n_x}$
represents the second parameter matrix encoding the stable weights $\mathbf{{A}}$. The lower and upper bound of the dominant eigenvalue are given by $\lambda_{\text{min}}$ and $\lambda_{\text{max}}$, respectively.

\paragraph{SVD weights:}
Inspired by singular value decomposition (SVD), this method decomposes a possibly non-square weight matrix $\mathbf{{A}} = \mathbf{U\boldsymbol{\Sigma}V}$ into two unitary matrices $\mathbf{U}$ and $\mathbf{V}$, and a diagonal matrix $\boldsymbol{\Sigma}$ with singular values on its diagonal.
The orthogonality of $\mathbf{U}$ and $\mathbf{V}$ is enforced via penalties: 
\begin{equation}
\mathcal{L}_{\text{reg}}  = || \mathbf{I} - \mathbf{UU}^\top ||_2 + || \mathbf{I} - \mathbf{U}^\top\mathbf{U} ||_2 
  +|| \mathbf{I} - \mathbf{VV}^\top ||_2 + || \mathbf{I} - \mathbf{V}^\top\mathbf{V} ||_2 
  \label{eq:SVD_reg}
\end{equation}
An alternative approach to penalties introduced in~\cite{zhang2018stabilizing} is to use Householder reflectors  to represent unitary matrices $\mathbf{U}$ and $\mathbf{V}$.
The constraints $\lambda_{\text{min}}$ and $\lambda_{\text{max}}$ on the singular values $\lambda$    can be  implemented  by clamping and scaling given as:
\begin{equation}
\mathbf{\boldsymbol{\Sigma}} = \text{diag}(\lambda_{\text{max}} - (\lambda_{\text{max}} - \lambda_{\text{min}}) \cdot \sigma(\lambda))
\end{equation}

\paragraph{GD weights:}
This method supporting square matrices leverages the {Gershgorin dics theorem}~\citep{Varga_Gersgorin2004}.
It says that
all eigenvalues $\lambda_i$ of the weight $ \mathbf{{A}}$  can be bounded in the complex plane with center $\lambda$ and radius $r$ given by the formula:
\begin{equation}
\label{eq:Gershgorin}
 \mathbf{{A}} = \texttt{diag}\begin{pmatrix}\frac{r}{s_1}, ..., \frac{r}{s_n}\end{pmatrix}\mathbf{M} +  \texttt{diag}\begin{pmatrix}\lambda, ..., \lambda\end{pmatrix}
\end{equation}
Where  
$\mathbf{M} \in  \mathbb{R}^{n\times n}$ with $m_{i,j} \sim \mathcal{U}(0,1)$, except $m_{i, i} = 0$ is learnable parameter matrix. While diagonal matrices $\texttt{diag}\begin{pmatrix}\frac{r}{s_1}, ..., \frac{r}{s_n}\end{pmatrix}$, and $\texttt{diag}\begin{pmatrix}\lambda, ..., \lambda\end{pmatrix}$ represent radii and centers of the bounded eigenvalues, where $s_j = \sum_{i\neq j} m_{i,j}$. 


\paragraph{Parametric stability constraints:} 
The disadvantage of enforcing the global stability conditions as given via Corollary~\ref{thm:DMM_stability_2}
is their negative effect on the expressivity of the DMM, resulting in dynamics with a single point or line attractors. This will effectively prevent the DMM from expressing  more complex attractors such limit cycles or chaotic attractors. 
As a more expressive alternative we introduce the use of parameter varying bounds in the conditions~\eqref{eq:sufficient_1}, and~\eqref{eq:sufficient_2}, such as:
\begin{equation}
   \underline{\mathbf{p}}(\mathbf{x})   < \|\mathbf{A}_{\mathbf{f}}(\mathbf{x}) \|_p <  \overline{\mathbf{p}}(\mathbf{x}) 
     \label{eq:parametric_bounds} 
\end{equation}
Where $ \underline{\mathbf{p}}(\mathbf{x}): \mathbb{R}^{n_{\mathbf{x}}} \to \mathbb{R}$, and $\overline{\mathbf{p}}(\mathbf{x}): \mathbb{R}^{n_{\mathbf{x}}} \to \mathbb{R}$ are scalar valued functions parametrizing lower and upper bounds of operator norm of the DMM's mean transition dynamics. 
Similar parametric constraints can be applied to the variance bounds in~\eqref{eq:sufficient_1}, or weight norm constraints in~\eqref{eq:sufficient_2}.
This approach allows us to control the contractivity of the DMMs depending on the position in the state space. This allows us to
 increase the expressivity of the DMM, e.g., by partitioning the state space into constrained and unconstrained regions resulting in DMM with hybrid or switching dynamics.
In particular, we could divide the state space to outer contractive regions (where conditions~\eqref{eq:sufficient_1} hold) and inner relaxed regions allowing for more complex trajectories to emerge. 
This parametrization will effectively generate non-empty attractor set in which it is possible to learn arbitrary attractor shape.
The proposed state space partitioning method is inspired by the Bendinxon-Dulac criteria on periodic solutions of differential equations~\citep{mccluskey_bendixson_dulac_1998}.


\section{Numerical Case Studies}
\label{sec:case_study}

 In this section we empirically validate the conditions given in Theorem~\ref{thm:DMM_stability_1} and Corollary~\ref{thm:DMM_stability_2}
 by investigating the dynamics of DMM's
  transition maps~\eqref{eq:DMM} whose mean $\mathbf{f}_{\theta_\mathbf{f}}(\mathbf{x})$ and variance $\mathbf{g}_{\theta_\mathbf{g}}(\mathbf{x})$ are parametrized by neural networks with different spectral distributions of their weights and activation scaling matrices~\eqref{eq:lambda_matrix}.
We apply spectral analysis to the PWA forms~\eqref{eq:lpv_dnn} of neural networks modeling the mean and variance maps to obtain the corresponding spectra of DMMs. We performed the experiments using the probabilistic programming language Pyro \citep{bingham2019pyro}.

\subsection{Design of the Experiments}
\label{sec:experiments}

Since the stability of DMM~\eqref{eq:DMM} 
depends on the transition dynamcis,  
in all of the case studies we consider a fully observable 
model with identity as an emission map.
We parametrize the mean and variances of the transition map $\mathbf{f}_{\theta_\mathbf{f}}(\mathbf{x})$~\eqref{eq:dnn_dyn} and $\mathbf{g}_{\theta_\mathbf{g}}(\mathbf{x})$~\eqref{eq:dnn_dyn2} by feedforward neural networks.
Given the mean neural network $\mathbf{f}_{\theta_\mathbf{f}}(\mathbf{x})$ we generate a set of different transition dynamcis by changing activation functions $\mathbf{v}(x) \in \{\texttt{ReLU}, \texttt{Tanh}, \texttt{Sigmoid}, \texttt{SELU}, \texttt{Softplus}\}$, layer depth $L \in \{ 1, 2 ,4, 8\} $, and presence of bias $b \in \{\texttt{True}, \texttt{False} \} $.
For the variance network $\mathbf{g}_{\theta_\mathbf{g}}(\mathbf{x})$
we use $\texttt{ReLU}$ activations.
For both, mean and variance networks we initialize their weights with desired spectral properties via design methods described in Section~\ref{sec:design}.
In particular we use SVD,  
  PF, and GD factorizations to bound the weight's singular values in a prescribed range. We generate three categories of weights $\mathbf{A}_i$: stable with operator norm strictly below one $||\mathbf{A}_i||_p < 1$, marginally stable with norm close one $||\mathbf{A}_i||_p \approx 1$, and unstable with norm larger than one $||\mathbf{A}_i||_p  > 1$.

\subsection{Stability Analysis of Deep Markov Models}
\label{sec:case_stability}
\begin{figure*}[htb!]
  \begin{center}
 {\begin{tabular}{rccc}
\rotatebox[origin=c]{90}{Stable mean}  \rotatebox[origin=c]{90}{$||\mathbf{A}_\mathbf{f}||_p  < 1$} & \includegraphics[width=.20\linewidth,valign=m,trim=60 25 100 00, clip]{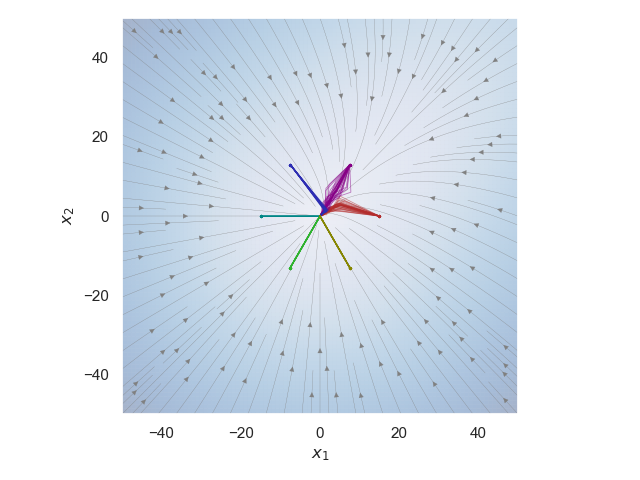}
& \includegraphics[width=.20\linewidth,valign=m,trim=60 25 100 00, clip]{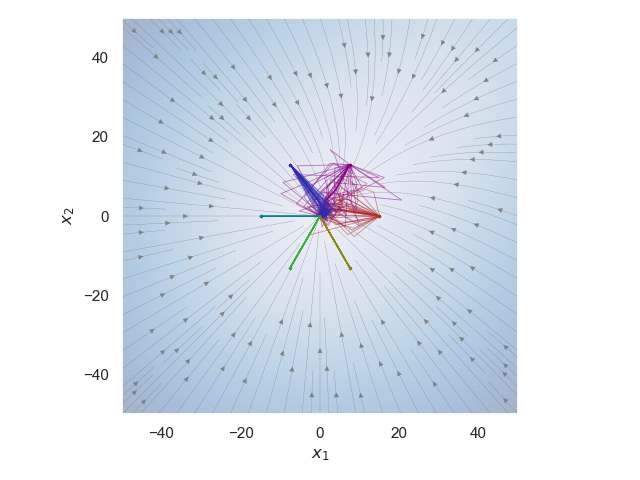} 
& \includegraphics[width=.20\linewidth,valign=m,trim=60 25 100 00, clip]{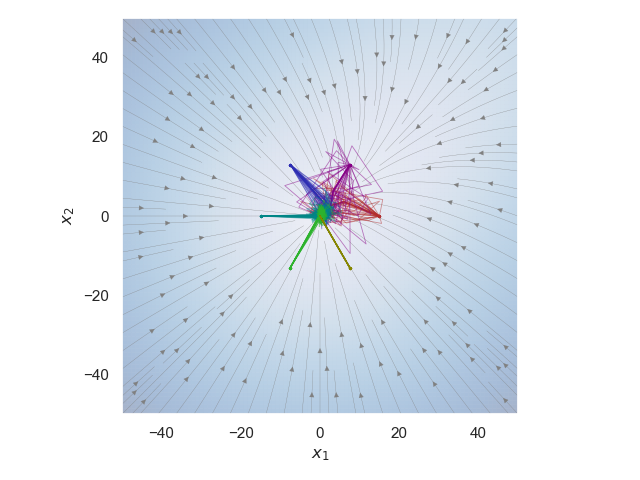}\\
\rotatebox[origin=c]{90}{Marginal  mean} \rotatebox[origin=c]{90}{$||\mathbf{A}_\mathbf{f}||_p  \approx 1$} & \includegraphics[width=.20\linewidth,valign=m,trim=60 25 100 00, clip]{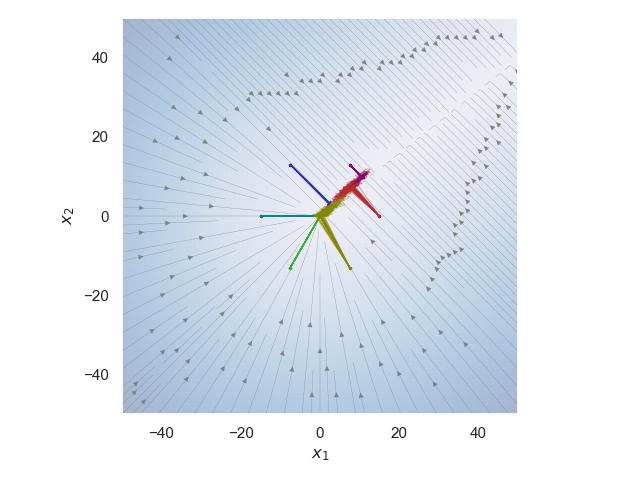} 
& \includegraphics[width=.20\linewidth,valign=m,trim=60 25 100 00, clip]{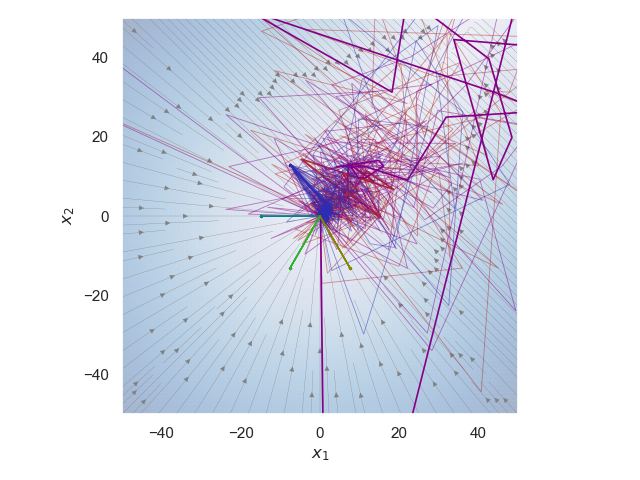} 
& \includegraphics[width=.20\linewidth,valign=m,trim=60 25 100 00, clip]{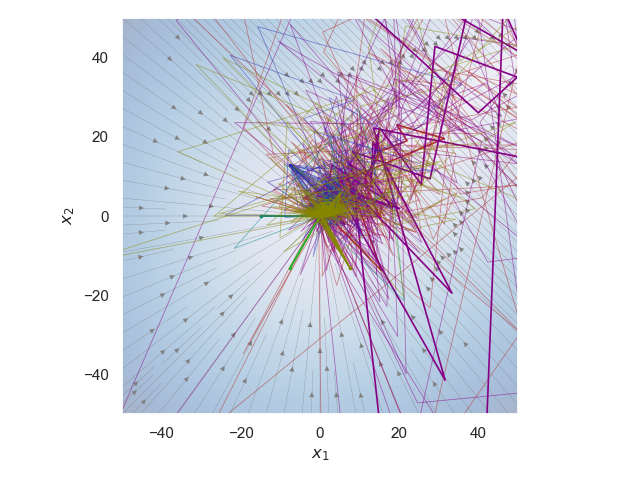}\\
\rotatebox[origin=c]{90}{Unstable mean} \rotatebox[origin=c]{90}{$||\mathbf{A}_\mathbf{f}||_p  > 1$} & \includegraphics[width=.20\linewidth,valign=m,trim=60 25 100 00, clip]{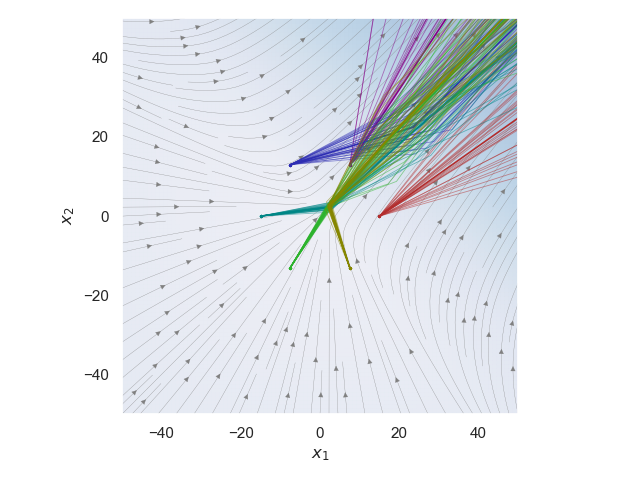} 
& \includegraphics[width=.20\linewidth,valign=m,trim=60 25 100 00, clip]{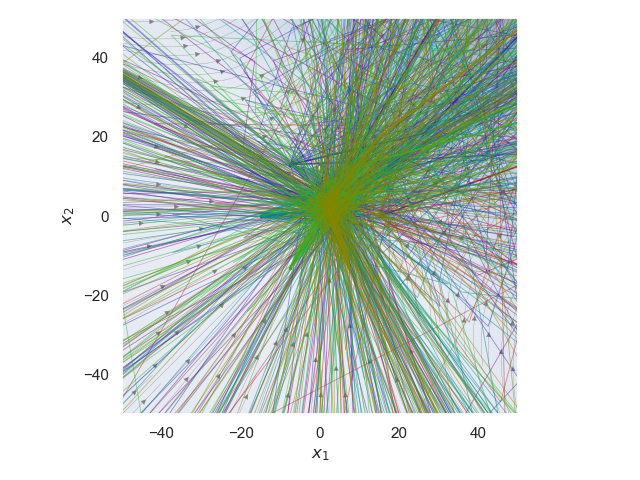}
& \includegraphics[width=.20\linewidth,valign=m,trim=60 25 100 00, clip]{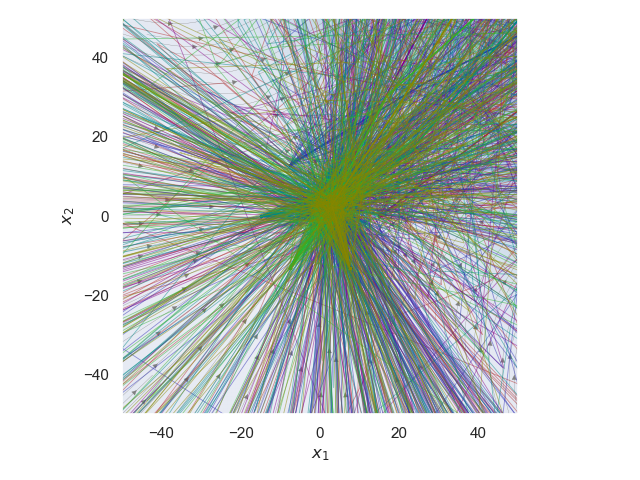}\\
&  \hspace{0.3cm} Stable  var. $\|\mathbf{A}_\mathbf{g}\|_p < 1$ 
&  \hspace{0.3cm} Marginal var. $\| \mathbf{A}_\mathbf{g}\|_p \approx 1$ 
&   \hspace{0.3cm} Unstable  var. $\|\mathbf{A}_\mathbf{g}\|_p > 1$ \\
\end{tabular}}
\caption{Phase portraits of DMMs demonstrating the effect of  norm bounds on mean $\mathbf{f}_{\theta_\mathbf{f}}(\mathbf{x})$ and variance $\mathbf{f}_{\theta_\mathbf{g}}(\mathbf{x})$ networks modeling transition dynamics. Thin lines are samples of the stochastic dynamics with bold lines representing mean trajectories. Colors represent different initial conditions.}
\label{fig:stability_phase_space}
  \end{center}
 \end{figure*}


In order to provide intuitive visualisations of the dynamics in the phase space, in this section, we focus on two dimensional system. 
 Fig.~\ref{fig:stability_phase_space}
 visualizes the phase portraits of randomly generated DMM's probabilistic transition maps of the mean $\mathbf{f}_{\theta_\mathbf{f}}(\mathbf{x})$ and variance $\mathbf{f}_{\theta_\mathbf{g}}(\mathbf{x})$ networks with  constrained operator norms enforced using PF weights from Section~\ref{sec:design}.
 Figures in the first row demonstrate that DMMs with asymptotically stable mean transition dynamics  $||\mathbf{A}_\mathbf{f}||_p < 1$ with bounded variances
 $\|\mathbf{A}_\mathbf{g}\|_p < K, \ K>0$ 
 generate
 stable single point attractors. Hence they 
 validate the sufficient conditions of Theorem~\ref{thm:DMM_stability_1}. 
%
 Figures in the second row display dynamics of DMM with marginally stable mean  $||\mathbf{A}_\mathbf{f}||_p \approx 1$.
 Due to the non-dissipativeness of the mean transition dynamics, the trajectories converge to a line attractor only if the variance is a converging map $\|\mathbf{A}_\mathbf{g} \|_p < 1$, thus having a dissipative second moment. In case of marginally stable variance, $\|\mathbf{A}_\mathbf{g} \|_p \approx 1$ the energy conserving nature of the mean and variance together generate random walk type trajectories along the direction of the mean's line attractor.
 For the cases with unstable variance $\|\mathbf{A}_\mathbf{g} \|_p > 1$, the overall dynamics behaves close to a Brownian motion with a degree of randomness, which is positively correlated with  the variance network's operator norm.
Figures in the third row show diverging dynamics of DMM with
unstable mean  $||\mathbf{A}_\mathbf{f}||_p > 1$.
With converging variance $\|\mathbf{A}_\mathbf{g} \|_p < 1$ the diverging stochastic trajectories stay close to the mean direction.
While, for both marginal and unstable variances
 the stochastic trajectories diverge in all directions.
%


\subsection{Effect of  Biases and Depth on the Stability of Deep Markov Models}

 \begin{figure*}[ht!]
   \begin{center}
 {\begin{tabular}{ccc}
    \includegraphics[width=.20\linewidth,valign=m,trim=60 25 100 00, clip]{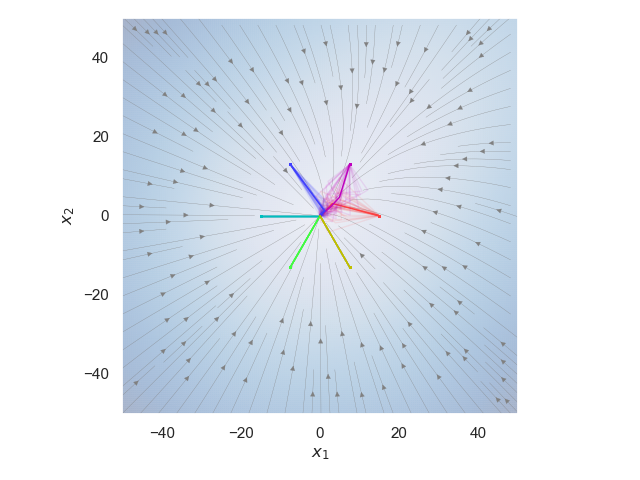}
    \rulesep
    & \includegraphics[width=.20\linewidth,valign=m,trim=60 25 90 00, clip]{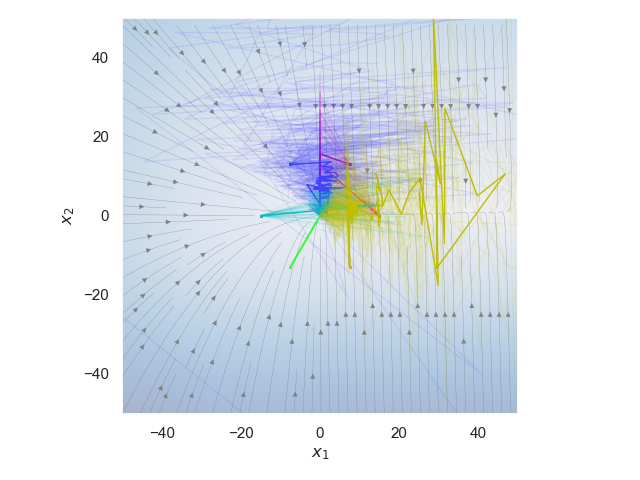}
    & \includegraphics[width=.20\linewidth,valign=m,trim=60 25 90 00, clip]{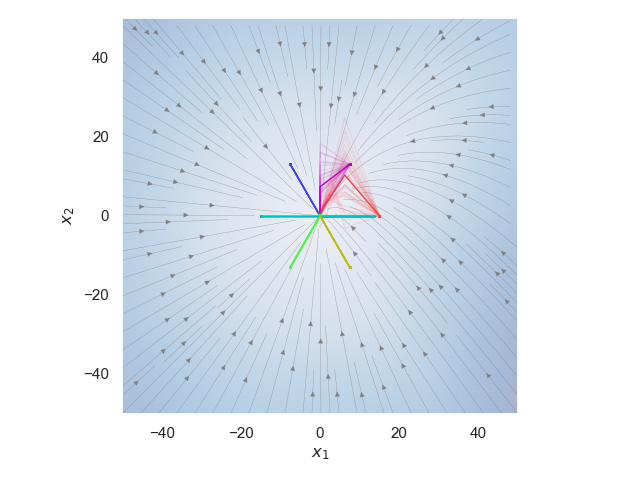}\\
\hspace{0.5cm} (a)  & 
\hspace{0.5cm} (b) &   \hspace{0.5cm} (c) \\
    \includegraphics[width=.20\linewidth,valign=m,trim=60 25 100 00, clip]{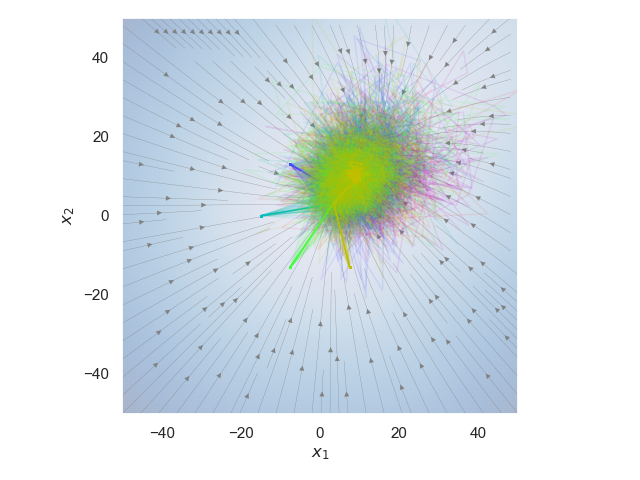}
    \rulesep
    &\includegraphics[width=.20\linewidth,valign=m,trim=60 25 100 00, clip]{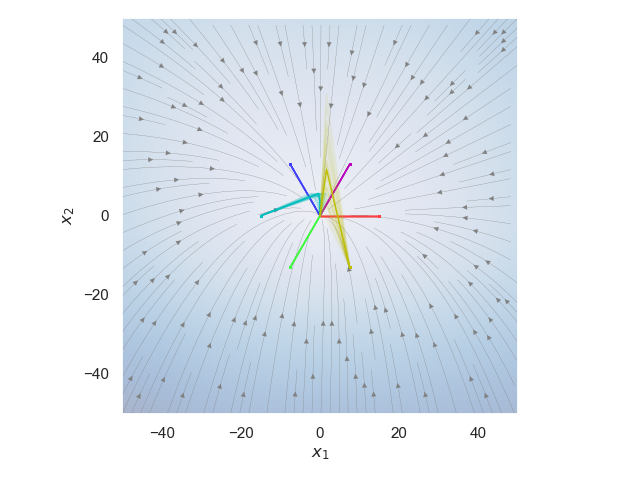}
    & \includegraphics[width=.20\linewidth,valign=m,trim=60 25 100 00, clip]{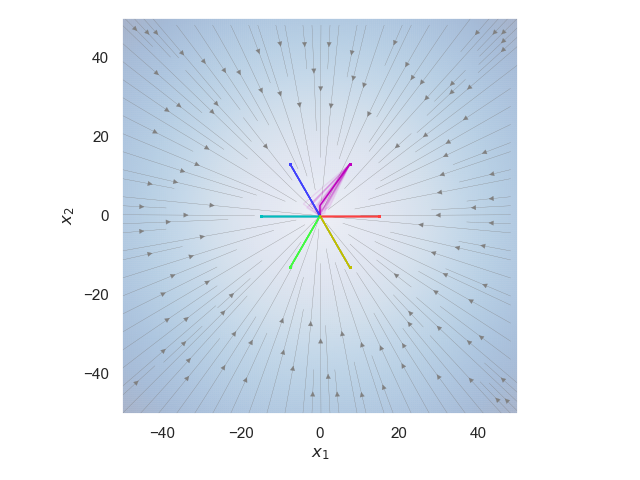}\\
\hspace{0.5cm} (d) & 
\hspace{0.5cm} (e) &   \hspace{0.5cm} (f) \\
\end{tabular}}
\caption{Left panels show the effect of biases using PF  regularization and \texttt{ReLU} activation ((a) w/o bias, (d)  w bias). Right panels show the effect of network  $\mathbf{f}$ depths with SVD regularization and \texttt{ReLU} : (b) $1$ layer, (c) $2$ layers, (e) $4$ layers, (f) $8$ layers. }
\label{fig: bias+depth}
  \end{center}
 \end{figure*}
 
 In Fig.~\ref{fig: bias+depth}, we experiment with  biases and depths of mean $\mathbf{f}_{\theta_\mathbf{f}}(\mathbf{x})$ and variance $\mathbf{g}_{\theta_\mathbf{g}}(\mathbf{x})$
 networks.
\paragraph{Effect of biases:}
 In the left panels of Fig.~\ref{fig: bias+depth}, we demonstrate the dynamics of DMM with \texttt{tanh} activations and SVD factorized weights resulting in asymptotically stable transition maps, thus $||\mathbf{A}_\mathbf{f}(\mathbf{x})|| < 1$, and $||\mathbf{A}_\mathbf{g}(\mathbf{x})|| < 1$. Fig.~\ref{fig: bias+depth}(a) shows the scenario without any bias whereas Fig.~\ref{fig: bias+depth}(d) shows the scenario where both $\mathbf{f}_{\theta_\mathbf{f}}(\mathbf{x}), \mathbf{g}_{\theta_\mathbf{g}}(\mathbf{x})$ have bias terms. It demonstrates that the general contractive nature of the stable behavior, as given via conditions~\eqref{eq:sufficient_1}, does not change with addition of biases. Instead, the biases shift the region of attraction by generating non-zero equilibrium points.  This shift is correlated with absolute value of the aggregate bias term of the PWA form~\eqref{eq:lpv_dnn}. 
For the norm bounds on the equilibria of stable DMM
see supplementary material.


 \paragraph{Effect of depth:}
The right panels  of Fig.~\ref{fig: bias+depth} demonstrate the dynamics of DMMs with increasing number of layers using  \texttt{ReLU} activations and SVD weights close to marginal stability $0.99 < ||\mathbf{A}_i(\mathbf{x})|| < 1$. It can be seen that with increase in the number of layers, the convergence of trajectories toward origin becomes less uncertain. In this case,  
 a larger number of mildly contractive layers 
 results in stabilizing behavior, demonstrating the effect of the norm submultiplicativity  on the operator norm of the mean transition dynamics~\eqref{eq:Gelfand_norm}.
 However, one needs to be careful about the delicate balance between stability, and ability to efficiently train the parameters of DMM with gradient-based optimization. 
 With increasing depth, the norm product of the contractive layers~\eqref{eq:Gelfand_norm}
 will eventually result in a network with a very small operator norm thus causing the vanishing gradient problem.
Analogously, the exploding gradient problem  will occur for DMM parametrized with very deep neural networks 
with non-contractive layers, i.e. $ ||\mathbf{A}_i(\mathbf{x})|| > 1$
The use of parametric stability constraints~\eqref{eq:parametric_bounds} as a function of depth could be an efficient strategy for avoiding the vanishing and exploding gradients by keeping the overall dynamics norm bounded.


\subsection{Deep Markov Models with Parametrized Stability Constraints}
\label{sec:case_hybrid}

In Fig.~\ref{fig:param_con} we demonstrate the use of parametrized stability constraints~\eqref{eq:parametric_bounds} in the design of stable DMMs without compromising the expressivity as it is in the case of restrictive single point attractors enforced via~\eqref{eq:sufficient_1} and~\eqref{eq:sufficient_2}.
In particular, we design two DMMs with randomly generated weights with three phase space regions with different mean transition dynamics, (i) an inner expanding region $||\mathbf{A}_\mathbf{f}(\mathbf{x})|| > 1, \ \mathbf{x}^i \in \mathcal{R}_1$, (ii) a middle marginal region $||\mathbf{A}_\mathbf{f}(\mathbf{x}^i)|| \approx 1, \ \mathbf{x}^i \in \mathcal{R}_2$, and (iii) an outer contractive region $||\mathbf{A}_\mathbf{f}(\mathbf{x}^i)|| < 1, \ \mathbf{x}^i \in \mathcal{R}_3$. 
Where the regions are given as $ \mathcal{R}_1 = \{\mathbf{x}| 0 \le ||\mathbf{x}||_2 <  20 \}$, $ \mathcal{R}_2 = \{\mathbf{x}| 20 \le ||\mathbf{x}||_2 <  40 \}$, and $ \mathcal{R}_3 = \{\mathbf{x}| 40 \le ||\mathbf{x}||_2 \}$, respectively. 
The variance dynamics in both cases is kept being contractive $||\mathbf{A}_\mathbf{g}(\mathbf{x}^i)|| < 1$.
From  Fig.~\ref{fig:param_con} it is apparent that the overall dynamics of the DMMs with parametrized  constraints~\eqref{eq:parametric_bounds} is able to generate stochastic periodic behavior
while remaining bounded within prescribed region of attraction, thus providing high degree of expressivity while being provable stable.
On the left (Fig.~\ref{fig:param_con} (a) and (c)) we show phase plots, and on the right (Fig.~\ref{fig:param_con} (b) and (d))  corresponding time series trajectories. 
As a potential extension, we envision learning
the constraints bounds $\underline{\mathbf{p}}(\mathbf{x})$, and $\overline{\mathbf{p}}(\mathbf{x}) $ in~\eqref{eq:parametric_bounds} using penalty methods.
 \begin{figure*}[ht!]
  \begin{center}
 {\begin{tabular}{cc}
    \includegraphics[width=.34\linewidth,valign=m,trim=60 25 80 0, clip]{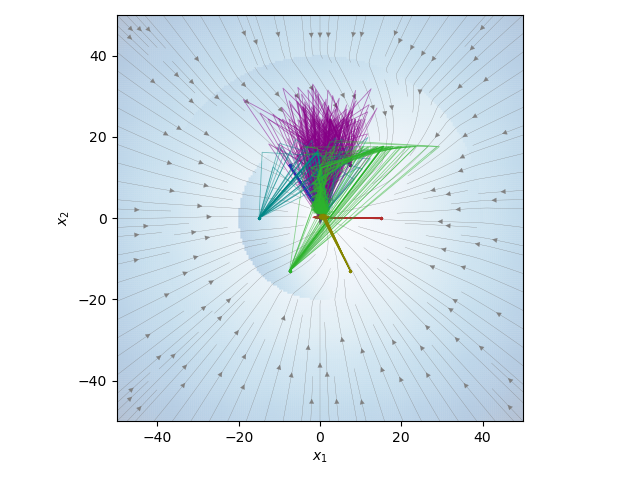}
    & \includegraphics[width=.42\linewidth,valign=m,trim=0 0 0 0, clip]{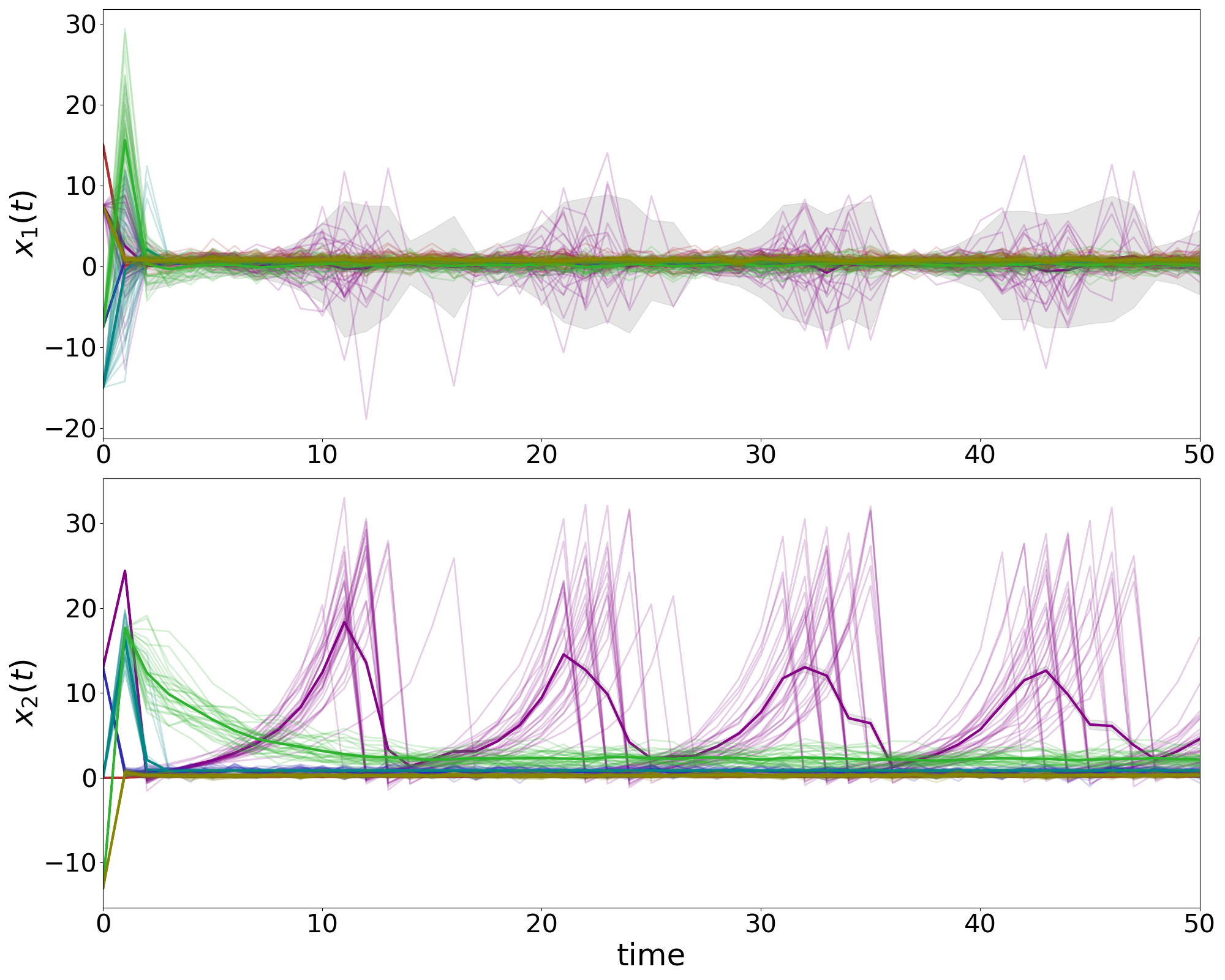}\\
\hspace{0.3cm} (a)  & 
\hspace{0.3cm} (b)  \\
    \includegraphics[width=.34\linewidth,valign=m,trim=60 25 80 0, clip]{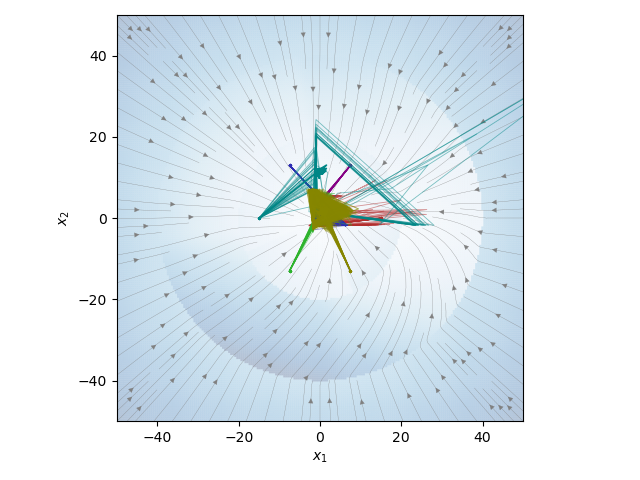}
    & \includegraphics[width=.42\linewidth,valign=m,trim=0 0 0 0, clip]{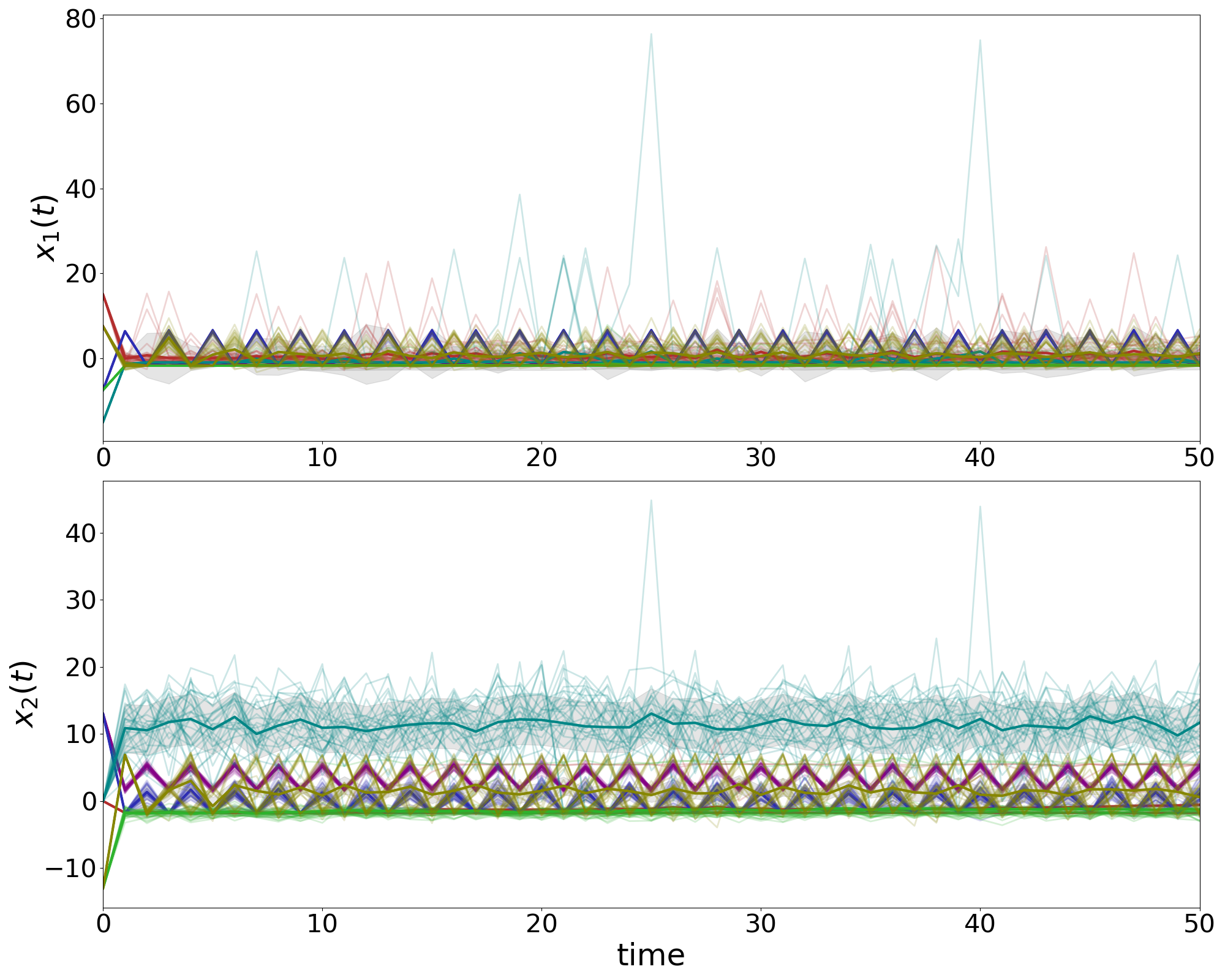}\\
\hspace{0.3cm} (c) & 
\hspace{0.3cm} (d) 
\end{tabular}}
\caption{Phase plots (left panels), and time series trajectories (right panels) of DMMs with parametrized stability constraints exhibiting periodic behavior within bounded regions of attaction. Both cases consider SVD regularizations with \texttt{softplus} (top) and \texttt{SELU}(bottom) respectively.}
\label{fig:param_con}
  \end{center}
 \end{figure*}
 

\section{Conclusion}
In this paper, we introduce a new stability analysis method for deep  Markov models (DMMs).
As the main result, we provide sufficient conditions for the stochastic stability and introduce a set of practical methods for designing provably stable DMMs. In particular, we discuss the use of contractive weight matrices factorizations and stability conditions for the activation functions. Furthermore,  we propose using novel parametric stability constraints allowing expression of more complex stochastic dynamics while remaining contractive towards non-empty region of attraction. 
The proposed theory is supported by numerical experiments, with design guidelines considering weight factorizations, choices of activation functions, network depth, or use of the bias terms for guaranteed stability.
In future work, we aim to derive the stability guarantees for a broader family of probability distributions modeled by normalizing flows.
We also aim to expand the theory
to partially observable Markov decision processed (POMDP) to derive closed-loop stability guarantees in the context of deep reinforcement learning. 




\begin{ack}

We acknowledge our colleagues Aaron Tuor, Mia Skomski, Soumya Vasisht, and Draguna Vrabie for their contributions to related topics that served as a base for developing the method presented in this paper.
We would like to thank Craig Baker and David Rolnick for fruitful discussions that helped improve the technical quality of the presented ideas.
Also, we want to thank our anonymous reviewers for their constructive feedback and suggestions.

This research was supported by the U.S. Department of Energy, through the Office of Advanced Scientific Computing Research's “Data-Driven Decision Control for Complex Systems (DnC2S)” project. Pacific Northwest National Laboratory is operated by Battelle Memorial Institute for the U.S. Department of Energy under Contract No. DE-AC05-76RL01830. Oak Ridge National Laboratory is operated by UT-Battelle LLC for the U.S. Department of Energy under contract number DE-AC05-00OR22725.

\end{ack}


\section{Limitations and Broader Impact}


Stability is the major concern of many safety-critical systems. The significance of the presented work lies in the proposed stability conditions, analysis, and design methods allowing the generation of provably stable DMMs. 
Furthermore, the presented methods for the design of stable DMMs could be of large significance in the context of stochastic control systems with stability guarantees.
Thus, the authors believe that besides academic relevance, the methods presented in this paper have the potential
for practical impact in many real-world applications such as
unmanned autonomous vehicles, robotics, or process control applications.

The authors are aware of the limitations of the presented numerical case studies focusing on small-scale DMMs with two-dimensional state space. This choice was made for the sake of the  visualizations of the state space trajectories 
allowing us to provide intuitive examples of the presented theoretical results. 
As part of the future work, the authors plan to use the proposed stability analysis and design methods for DMMs on real-world datasets with higher dimensional state space.

The presented work falls into the basic research category. As such, authors are not aware of any potential direct negative societal impact of the proposed work.
On the contrary, authors of this paper believe that the presented theory is a minor contribution
towards general knowledge, which accumulation has been historically proved to inherently benefit all humanity.

\appendix

\section{Supplementary Material}

\subsection{Background}
\label{sec:prelim}
In this section, we recall few important mathematical definitions and theorems used in the paper. 
\begin{definition}
Induced operator norm of a matrix $\mathbf{A} \in \mathbb{R}^{n\times m}$ is defined as:
 \begin{equation}
 ||\mathbf{A}||_p =  \max_{\mathbf{x} \neq 0} \frac{ ||\mathbf{A}\mathbf{x}||_p}{||\mathbf{x}||_p}  =   \max_{\|\mathbf{x}\|_p = 1} ||\mathbf{A}\mathbf{x}||_p, \ \ \forall \mathbf{x} \in \mathcal{X},
   \label{eq:operator_norm}
 \end{equation}
 where $\mathcal{X}$ is a compact normed vector space, and  $|| \cdot ||_p: \mathbb{R}^{n} \to \mathbb{R}$ represents vector $p$-norm inducing the matrix norm $ || \mathbf{A} ||_p: \mathbb{R}^{n \times m} \to \mathbb{R}$. 
 In case of Euclidean norm $|| \cdot ||_2$, the operator norm
corresponds to the largest singular value $ \sigma_{\mbox{max}}(\mathbf{A})$:
 \begin{equation}
  \label{eq:operator_norm_2}
|| \mathbf{A} ||_2 = \max_{||\mathbf{x}||_2 = 1}  ||\mathbf{A} \mathbf{x}||_2 = \sigma_{\mbox{max}}(\mathbf{A}).
\end{equation}

 The matrix norm is sub-additive:
  \begin{equation}
  \label{eq:operator_norm_subadd}
|| \mathbf{A} + \mathbf{B} ||_p \le || \mathbf{A}||_p + ||\mathbf{B} ||_p.
\end{equation} 

Induced $p$-norm $||\cdot||_p: \mathbb{R}^{n \times m} \to \mathbb{R}$ is called submultiplicative if it satisfies~\citep{matrix_norms1983}:
 \begin{equation}
  \label{eq:operator_norm_submultiplicative}
|| \mathbf{A} \mathbf{B} ||_p \le || \mathbf{A}||_p  ||\mathbf{B} ||_p.
\end{equation} 
\end{definition}

\begin{remark}
  For commonly used norms $||\cdot||_p$ where $p \in \{1, 2, \infty \}$ the submultiplicativity~ \eqref{eq:operator_norm_submultiplicative} holds.
\end{remark}

\begin{definition}
Asymptotic stability of the dynamical system 
for a bounded initial condition $\mathbf{x}_{0}$ implies that its states converge to the equilibrium point $\bar{\mathbf{x}}_e$:
  \label{def:asymptotic}
  \begin{equation}
  \label{eq:asymptotic}
     || \mathbf{x}_0 - \bar{\mathbf{x}}_e || < \delta  \implies \lim_{t \rightarrow \infty} ||\bar{{\mathbf{x}}}_t|| =  \overline{\mathbf{x}}_e
 \end{equation}
\end{definition}

\begin{definition}
Given a metric space $(\mathcal{X},d)$, a mapping $T: \mathcal{X}\to \mathcal{X}$ is called contractive if there exist a constant $c \in [0, 1)$ and a metric $d$ such that following holds:
 \begin{equation}
  \label{eq:contraction}
      d(T( \mathbf{x_1}),T( \mathbf{x_2})) \le  c d( \mathbf{x_1}, \mathbf{x_2}), \  \forall  \mathbf{x_1},  \mathbf{x_2} \in   \mathcal{X}
\end{equation} 
Contraction generalizes to
maps $T: \mathcal{X}\to \mathcal{Y}$ between two metric spaces  $(\mathcal{X},d)$, and  $(\mathcal{Y},d')$ as:
 \begin{equation}
  \label{eq:contraction_2}
      d'(T( \mathbf{x_1}),T( \mathbf{x_2})) \le  c d( \mathbf{x_1}, \mathbf{x_2}), \  \forall  \mathbf{x_1},  \mathbf{x_2} \in   \mathcal{X}
\end{equation} 
\end{definition}

\begin{definition}
\label{def:affine_contract}
An affine map $T(\mathbf{x}) =  \mathbf{A} \mathbf{x} + \mathbf{b}$ with $\mathbf{x} \in \mathbb{R}^n$ and  metric $d = ||\cdot||_p$ is a contraction mapping 
if the spectral norm of its linear component $\mathbf{A}$  is less than one, i.e. $ ||\mathbf{A}||_2 < 1$.
\end{definition}

\begin{theorem}
\label{thm:banch}
Banach fixed-point theorem. Lets have non-empty complete metric space $(\mathcal{X},d)$ then every contraction~\eqref{eq:contraction} is converging towards a unique fixed point  $T( \mathbf{x}_{\text{ss}}) = \mathbf{x}_{\text{ss}}$.
\end{theorem}

\subsection{Stable Deep Markov Models with Bounded Equilibria}
\label{sec:corollary6}

\begin{corollary}
\label{thm:DMM_stability_3}
An equilibrium $\bar{\mathbf{x}}_e$ of a mean-square stable
deep Markov model is bounded by $\underline{\mathbf{x}} \le || \bar{\mathbf{x}}_e ||_p \le \overline{\mathbf{x}}$ if either of the conditions in Theorem 3
or Corollalry 5 are satisfied,
and the following holds $\forall \mathbf{x} \in \text{dom}(\mathbf{f}_{\theta_\mathbf{f}}(\mathbf{x}))$.
\begin{equation}
\label{eq:stable_eq_dmm}
  \overline{\mathbf{x}} \le \frac{ \|\mathbf{b}_{\mathbf{f}}({\mathbf{x}})\|_p }{1- \|\mathbf{A}_{\mathbf{f}}({\mathbf{x}}) \|_p}, \ \
 \frac{ \|\mathbf{b}_{\mathbf{f}}({\mathbf{x}})\|_p }{1+ \|\mathbf{A}_{\mathbf{f}}({\mathbf{x}}) \|_p} \le     \underline{\mathbf{x}} 
\end{equation}
\end{corollary}


\begin{proof}
The conditions in Theorem 3
and Corollary 5 imply 
mean square stability of the DMM with a mean dynamics converging towards a stable equilibrium $\bar{\mathbf{x}}_e$:
\begin{equation}
\label{eq:DMM_equilibrium}
  \bar{\mathbf{x}}_e = \mathbf{f}_{\theta_\mathbf{f}}(\bar{\mathbf{x}}_e) = 
   \lim_{t \to \infty} \mathbf{f}_{\theta_\mathbf{f}}(\bar{\mathbf{x}}_t)
\end{equation}
Now by substitution of the PWA form
of the DNN into~\eqref{eq:DMM_equilibrium}
we get:
\begin{equation}
\label{eq:DMM_equilibrium_pwa}
  \bar{\mathbf{x}}_e = \mathbf{A}_{\mathbf{f}}(\bar{\mathbf{x}}_e) \bar{\mathbf{x}}_e + \mathbf{b}_{\mathbf{f}}(\bar{\mathbf{x}}_e) 
\end{equation}
Where the matrix  $\mathbf{A}_{\mathbf{f}}(\bar{\mathbf{x}}_e)$ and bias vector $\mathbf{b}_{\mathbf{f}}(\bar{\mathbf{x}}_e)$ uniquely define the affine equilibrium dynamics.

For the upper bound in~\eqref{eq:stable_eq_dmm},
we apply operator norm~\eqref{eq:operator_norm}
to the equation~\eqref{eq:DMM_equilibrium_pwa} to obtain:
\begin{equation}
\label{eq:DMM_eq_norm}
 || \bar{\mathbf{x}}_e ||_p = ||  \mathbf{A}_{\mathbf{f}}(\bar{\mathbf{x}}_e) \bar{\mathbf{x}}_e + \mathbf{b}_{\mathbf{f}}(\bar{\mathbf{x}}_e) ||_p
\end{equation}
Then applying triangle inequality~\eqref{eq:operator_norm_subadd} and operator upper bound 
$ ||\mathbf{A} \mathbf{x}||_p \le  ||\mathbf{A}||_p ||\mathbf{x}||_p $
we get:
\begin{equation}
\label{eq:DMM_eq_norm_subadd}
 || \bar{\mathbf{x}}_e ||_p \le ||  \mathbf{A}_{\mathbf{f}}(\bar{\mathbf{x}}_e) ||_p  ||\bar{\mathbf{x}}_e ||_p + || \mathbf{b}_{\mathbf{f}}(\bar{\mathbf{x}}_e) ||_p
\end{equation}
And by applying straightforward algebra we have:
\begin{equation}
\label{eq:DMM_eq_norm_subadd_2}
 (1-||  \mathbf{A}_{\mathbf{f}}(\bar{\mathbf{x}}_e) ||_p ) || \bar{\mathbf{x}}_e ||_p \le   || \mathbf{b}_{\mathbf{f}}(\bar{\mathbf{x}}_e) ||_p
\end{equation}
With resulting equilibrium upper bound given as:
\begin{equation}
\label{eq:DMM_eq_upper}
    ||\bar{\mathbf{x}}_e ||_p \le \frac{||\mathbf{b}_{\mathbf{f}}(\bar{\mathbf{x}}_e) ||_p}{ 1- || \mathbf{A}_{\mathbf{f}}(\bar{\mathbf{x}}_e) ||_p  }
\end{equation}

For deriving the lower bound in~\eqref{eq:stable_eq_dmm}, we start  
with straightforward algebraic operations on~\eqref{eq:DMM_equilibrium_pwa} to obtain equilibrium dynamics in a form:
\begin{equation}
\label{eq:DMM_equilibrium_dyn}
 (\mathbf{I} - \mathbf{A}_{\mathbf{f}}(\bar{\mathbf{x}}_e)) \bar{\mathbf{x}}_e =  \mathbf{b}_{\mathbf{f}}(\bar{\mathbf{x}}_e) 
\end{equation}
Again for two equivalent vectors their  norms must be equal:
\begin{equation}
\label{eq:DMM_equ_norm}
  ||(\mathbf{I} - \mathbf{A}_{\mathbf{f}}(\bar{\mathbf{x}}_e)) \bar{\mathbf{x}}_e ||_p  = ||\mathbf{b}_{\mathbf{f}}(\bar{\mathbf{x}}_e) ||_p
\end{equation}
Now applying operator norm upper bound inequality 
$ ||\mathbf{A} \mathbf{x}||_p \le  ||\mathbf{A}||_p ||\mathbf{x}||_p $ to~\eqref{eq:DMM_equ_norm}
we have:
\begin{align}
\label{eq:DMM_equilibrium_pwa_norm}
      ||(\mathbf{I} - \mathbf{A}_{\mathbf{f}}(\bar{\mathbf{x}}_e)) ||_p ||\bar{\mathbf{x}}_e ||_p  \ge ||\mathbf{b}_{\mathbf{f}}(\bar{\mathbf{x}}_e) ||_p \\
      ||\bar{\mathbf{x}}_e ||_p \ge \frac{||\mathbf{b}_{\mathbf{f}}(\bar{\mathbf{x}}_e) ||_p}{  ||\mathbf{I} - \mathbf{A}_{\mathbf{f}}(\bar{\mathbf{x}}_e) ||_p  }
\end{align}
Then from triangle inequality $||\mathbf{I} - \mathbf{A}_{\mathbf{f}}(\bar{\mathbf{x}}_e) ||_p  \le ||\mathbf{I}|| + ||\mathbf{A}_{\mathbf{f}}(\bar{\mathbf{x}}_e) ||_p$
we obtain a lower bound on the equilibrium norm as follows:
\begin{equation}
\label{eq:DMM_eq_lower}
    ||\bar{\mathbf{x}}_e ||_p \ge \frac{||\mathbf{b}_{\mathbf{f}}(\bar{\mathbf{x}}_e) ||_p}{ 1+ || \mathbf{A}_{\mathbf{f}}(\bar{\mathbf{x}}_e) ||_p  }
\end{equation}
Now clearly if the conditions of Corollary 5 are satisfied then
the  conditions~\eqref{eq:DMM_eq_upper} and~\eqref{eq:DMM_eq_lower}  hold.
\end{proof}


\subsection{Stability Regularizations for Deep Markov Models}

This section proposes additional  regularization methods  for learning stable deep Markov models.
The most direct approach is to include the stability conditions as extra penalties in the DMM loss function.
\begin{equation}
\label{eq:reg:stable1}
    \mathcal{L}_{\texttt{stable}} = \\
 \max\big(1, \|\mathbf{A}_{\mathbf{f}}({\mathbf{x}}) \|_p\big) 
 + \max\big(K, ||\mathbf{A}_{\mathbf{g}}(\mathbf{x})||_p+ \frac{||\mathbf{b}_{\mathbf{g}}(\mathbf{x})||_p}{\| \mathbf{x} \|_p} \big) 
\end{equation}
Then the stability will be enforced by assigning large penalty weights to~\eqref{eq:reg:stable1}.
Additionally, we can bound the norms of the equilibria of DMMs by enforcing constraints~\eqref{eq:stable_eq_dmm} via following penalties.
\begin{equation}
\label{eq:reg:stable2}
    \mathcal{L}_{\texttt{bounded}} = \\
 \max\big(0, \underline{\mathbf{x}} - \frac{ \|\mathbf{b}_{\mathbf{f}}({\mathbf{x}})\|_p }{1- \|\mathbf{A}_{\mathbf{f}}({\mathbf{x}}) \|_p}   \big) 
 + \max\big(0, \frac{ \|\mathbf{b}_{\mathbf{f}}({\mathbf{x}})\|_p }{1- \|\mathbf{A}_{\mathbf{f}}({\mathbf{x}}) \|_p} - \overline{\mathbf{x}}  \big) 
\end{equation}
We leave the empirical validation of the proposed regularizations~\eqref{eq:reg:stable1} and~\eqref{eq:reg:stable2} for future work.

\subsection{Effect of Activation Functions on Stability of Deep Markov Models}
\label{sec:activations}


As demonstrated by~\citep{massaroli2021dissecting}, different activations generate different phase fields in the context of neural ODEs.
In a similar spirit we explore the phase space behaviors of DMMs using different activations and weight regularizations. We kept the weights to be marginally stable $||\mathbf{A}_i|| \approx 1$, and thereafter vary the activation functions and weight regularization type. In the top row of Fig.~\ref{fig: activations}, we consider SVD based regularization for a $2-$layer DNN $\mathbf{f}_{\theta_\mathbf{f}}(\mathbf{x})$ for mean, and $3-$layer DNN $\mathbf{g}_{\theta_\mathbf{g}}(\mathbf{x})$ for diagonal covariances, and test with \texttt{ReLU, SELU, Softplus}. We can see that both \texttt{ReLU} and \texttt{SELU} produce sufficiently stable behaviors with few state trajectories for \texttt{SELU} remain oscillating near the equilibrium. Similar to \texttt{SELU}, state excursion generated by \texttt{Softplus} networks remain bounded but with higher uncertainty. In the bottom row of Fig.~\ref{fig: activations} we consider tight weight eigenvalue constraints based on Greshgorin discs. For the \texttt{ReLU}, and \texttt{Softplus}, along with the origin, both the axes become attractors, and \texttt{Softplus} produces higher uncertain trajectory oscillations. On the other hand, due to bounded (and contractive) tails \texttt{tanh} produces much more stable behavior with Greshgorin factorization. The behavior of different activations can be explained by the contractivity conditions given as $||\boldsymbol\Lambda^{\mathbf{f}}_{\mathbf{z}_i}||_p \le 1$.
The contractivity of activation functions is uniquely defined by their Lipschitz constants.
For instance, we know that functions with trivial nullspace and Lipschitz constant $\mathcal{K} < 1$, 
such as \texttt{ReLU, tanh},
are asymptotically stable and hence show less uncertain state excursions. However, activations such as \texttt{SELU}, or {Softplus} are not contractive over the entire domain, and therefore, they can generate unstable  dynamics even in the case of contractive weights. 
  \begin{figure*}[ht!]
  \begin{center}
 {\begin{tabular}{ccc}
\includegraphics[width=.20\linewidth,valign=m,trim=60 25 100 00, clip]{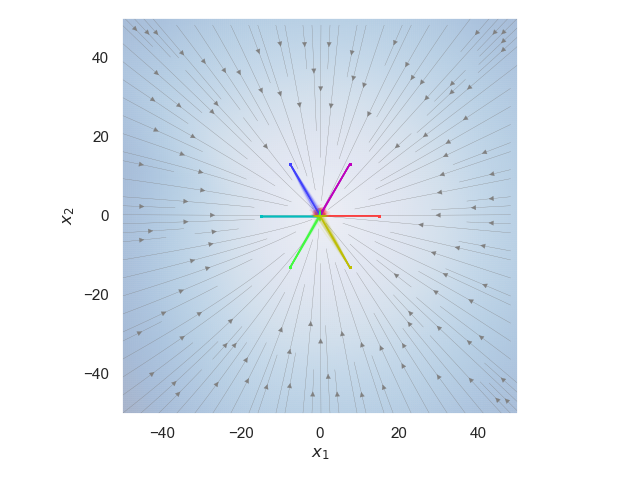} 
& \includegraphics[width=.20\linewidth,valign=m,trim=60 25 100 00, clip]{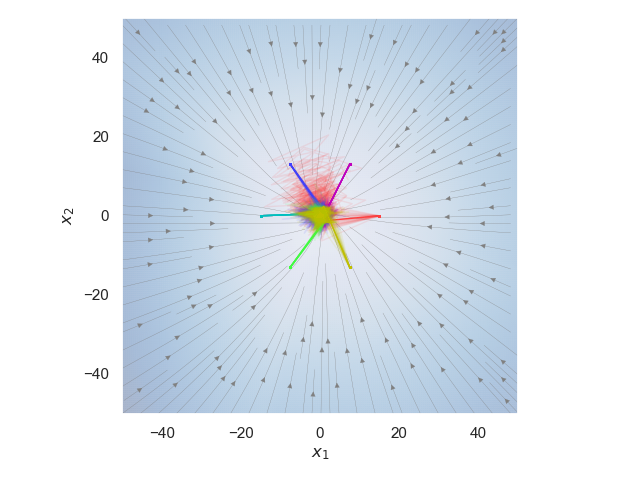} 
& \includegraphics[width=.20\linewidth,valign=m,trim=60 25 100 00, clip]{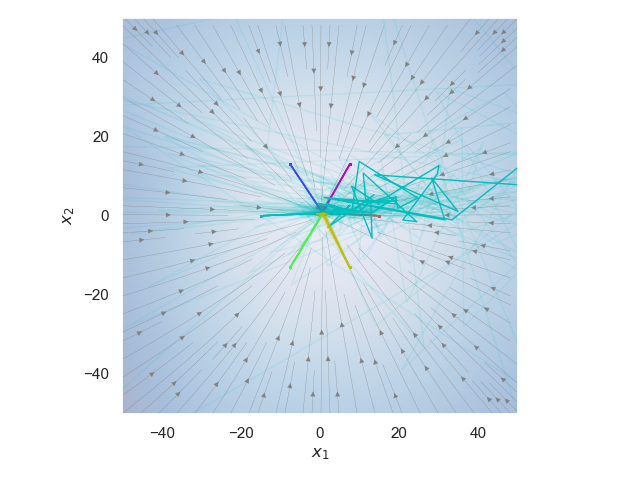}\\
\hspace{0.5cm} SVD and \texttt{ReLU} & 
\hspace{0.3cm} SVD and \texttt{SELU} &   \hspace{0.3cm} SVD and \texttt{Softplus} \\
\includegraphics[width=.20\linewidth,valign=m,trim=60 25 100 00, clip]{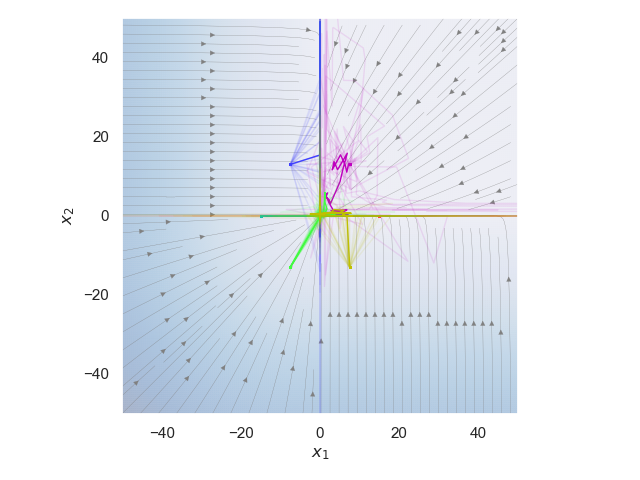}
& \includegraphics[width=.20\linewidth,valign=m,trim=60 25 100 00, clip]{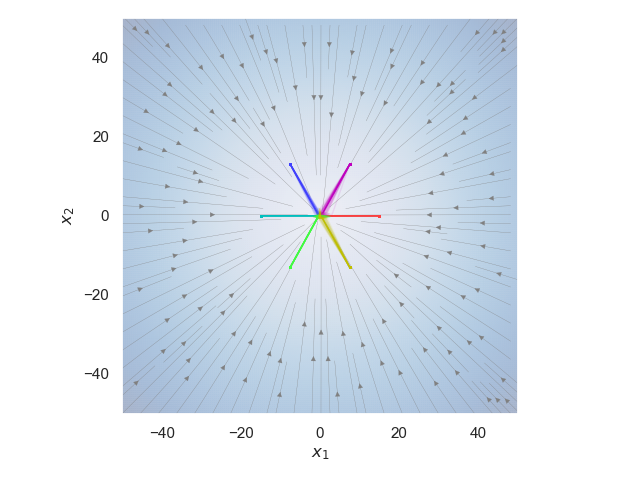} 
& \includegraphics[width=.20\linewidth,valign=m,trim=60 25 100 00, clip]{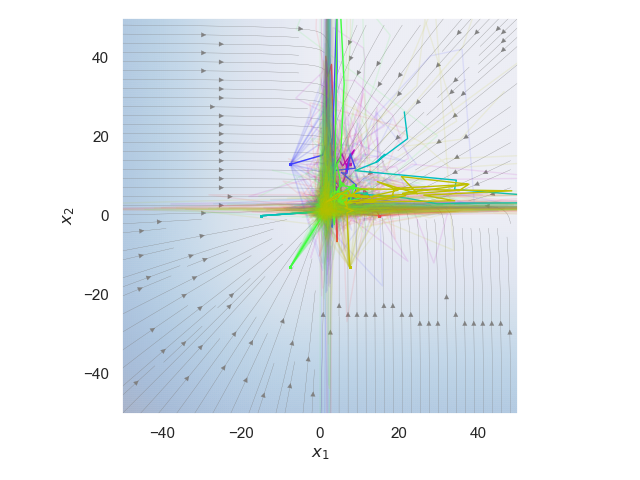}\\
 Greshg. and \texttt{ReLU}  &  \hspace{0.3cm} Greshg.  and \texttt{tanh} &  \hspace{0.3cm}  Greshg. and \texttt{Softplus} \\
\end{tabular}}
\caption{Experiment with different activations for marginally stable weight regularizations using SVD-based and Greshgorin disc-based factorizations. Thin lines are different realizations of the stochastic dynamics with bold lines being their mean trajectory.}
\label{fig: activations}
\end{center}
\end{figure*}

\bibliography{refs}
\bibliographystyle{abbrvnat}
\end{document}